\documentclass[letterpaper, 10 pt, conference]{ieeeconf}
\IEEEoverridecommandlockouts
\usepackage{amsmath,amsfonts,amssymb}
\newenvironment{IEEEproof}{\par\noindent\textit{Proof:}\ \ignorespaces}%
                          {\unskip\nobreak\hfill$\square$\par\medskip}
\usepackage{array}
\usepackage{booktabs}
\usepackage{textcomp}
\usepackage{stfloats}
\usepackage{url}
\usepackage{graphicx}
\usepackage{cite}
\newtheorem{theorem}{Theorem}

\newtheorem{assumption}{Assumption}
\newtheorem{remark}{Remark}

\providecommand{\N}[1]{\ifcsname pa:#1\endcsname\csname pa:#1\endcsname\else\textbf{??#1??}\errmessage{paperA: undefined number #1}\fi}
\expandafter\newcommand\csname pa:A_lin_window\endcsname{0.880}
\expandafter\newcommand\csname pa:As_rule_freach\endcsname{7.00}
\expandafter\newcommand\csname pa:G_locked\endcsname{50000}
\expandafter\newcommand\csname pa:G_peak\endcsname{85000}
\expandafter\newcommand\csname pa:J_m_eff\endcsname{1.83}
\expandafter\newcommand\csname pa:J_m_eff_hr\endcsname{0.034}
\expandafter\newcommand\csname pa:K_I_max\endcsname{1245}
\expandafter\newcommand\csname pa:a05c_filt2_hole_hi\endcsname{2.44}
\expandafter\newcommand\csname pa:a05c_filt2_hole_lo\endcsname{1.47}
\expandafter\newcommand\csname pa:a05s_edge_hi\endcsname{20}
\expandafter\newcommand\csname pa:a05s_edge_lo\endcsname{17.5}
\expandafter\newcommand\csname pa:alloc_f0dB_mu\endcsname{10.5}
\expandafter\newcommand\csname pa:alloc_f1dB_mu\endcsname{6.59}
\expandafter\newcommand\csname pa:alloc_f3dB_mu\endcsname{5.13}
\expandafter\newcommand\csname pa:alloc_f3dB_ratio\endcsname{2.40}
\expandafter\newcommand\csname pa:alloc_model_f3dB_mu\endcsname{12.3}
\expandafter\newcommand\csname pa:alloc_model_floor\endcsname{0.00048}
\expandafter\newcommand\csname pa:alloc_model_fx\endcsname{18.1}
\expandafter\newcommand\csname pa:alloc_model_fx_ReHalf\endcsname{18.1}
\expandafter\newcommand\csname pa:alloc_mu_amp_fmax\endcsname{15.9}
\expandafter\newcommand\csname pa:alloc_mu_amp_max\endcsname{3.13}
\expandafter\newcommand\csname pa:alloc_ratio_fmax\endcsname{15.9}
\expandafter\newcommand\csname pa:alloc_ratio_max\endcsname{0.896}
\expandafter\newcommand\csname pa:bk_over_A05\endcsname{18.1}
\expandafter\newcommand\csname pa:blt_G_clamp\endcsname{50000}
\expandafter\newcommand\csname pa:blt_G_peak\endcsname{85000}
\expandafter\newcommand\csname pa:blt_Jm_ratio\endcsname{1.06}
\expandafter\newcommand\csname pa:blt_Kt_ratio\endcsname{0.827}
\expandafter\newcommand\csname pa:blt_T_mu_req_clamp\endcsname{6.83}
\expandafter\newcommand\csname pa:blt_T_mu_req_peak\endcsname{4.62}
\expandafter\newcommand\csname pa:blt_acc_clamp\endcsname{64.1}
\expandafter\newcommand\csname pa:blt_acc_peak\endcsname{109}
\expandafter\newcommand\csname pa:blt_f_reach_17_peak\endcsname{16.1}
\expandafter\newcommand\csname pa:blt_f_reach_3_clamp\endcsname{13.5}
\expandafter\newcommand\csname pa:blt_f_reach_3_peak\endcsname{17.5}
\expandafter\newcommand\csname pa:ceil_asym_err_max\endcsname{10.1}
\expandafter\newcommand\csname pa:ceil_asym_err_min\endcsname{1.41}
\expandafter\newcommand\csname pa:ceil_c50_r15\endcsname{0.505}
\expandafter\newcommand\csname pa:ceil_c50_r20\endcsname{0.519}
\expandafter\newcommand\csname pa:ceil_c75_r15\endcsname{0.764}
\expandafter\newcommand\csname pa:ceil_c75_r20\endcsname{0.770}
\expandafter\newcommand\csname pa:ceil_macro_f15_errpct\endcsname{3.75}
\expandafter\newcommand\csname pa:ceil_macro_f15_meas\endcsname{5.84}
\expandafter\newcommand\csname pa:ceil_macro_f15_theory\endcsname{5.63}
\expandafter\newcommand\csname pa:ceil_macro_f20_errpct\endcsname{1.41}
\expandafter\newcommand\csname pa:ceil_macro_f20_meas\endcsname{3.21}
\expandafter\newcommand\csname pa:ceil_macro_f20_theory\endcsname{3.17}
\expandafter\newcommand\csname pa:ceil_npts\endcsname{9}
\expandafter\newcommand\csname pa:ceil_pisea_f15_errpct\endcsname{4.73}
\expandafter\newcommand\csname pa:ceil_pisea_f15_meas\endcsname{5.90}
\expandafter\newcommand\csname pa:ceil_pisea_f15_theory\endcsname{5.63}
\expandafter\newcommand\csname pa:ceil_pisea_f20_errpct\endcsname{3.23}
\expandafter\newcommand\csname pa:ceil_pisea_f20_meas\endcsname{3.27}
\expandafter\newcommand\csname pa:ceil_pisea_f20_theory\endcsname{3.17}
\expandafter\newcommand\csname pa:ceil_slope_c100\endcsname{0.550}
\expandafter\newcommand\csname pa:ceil_slope_c50\endcsname{0.565}
\expandafter\newcommand\csname pa:ceil_slope_c75\endcsname{0.554}
\expandafter\newcommand\csname pa:ceil_slope_theory\endcsname{0.562}
\expandafter\newcommand\csname pa:ceil_spring_err_max\endcsname{1.84}
\expandafter\newcommand\csname pa:ceil_spring_err_min\endcsname{-4.01}
\expandafter\newcommand\csname pa:ceiling_f10\endcsname{12.7}
\expandafter\newcommand\csname pa:ceiling_f15\endcsname{5.63}
\expandafter\newcommand\csname pa:ceiling_f20\endcsname{3.17}
\expandafter\newcommand\csname pa:des_A_macro_spec\endcsname{5.49}
\expandafter\newcommand\csname pa:des_G\endcsname{86735}
\expandafter\newcommand\csname pa:des_T_mu_req\endcsname{4.51}
\expandafter\newcommand\csname pa:des_acc\endcsname{116}
\expandafter\newcommand\csname pa:des_f_reach_cont\endcsname{16.7}
\expandafter\newcommand\csname pa:des_f_reach_peak\endcsname{24.4}
\expandafter\newcommand\csname pa:des_margin_peak\endcsname{1.40}
\expandafter\newcommand\csname pa:eps\endcsname{0.024}
\expandafter\newcommand\csname pa:eps_clamped\endcsname{0.016}
\expandafter\newcommand\csname pa:eps_ratio_torque\endcsname{0.257}
\expandafter\newcommand\csname pa:epsstar_torque_L1\endcsname{10.8}
\expandafter\newcommand\csname pa:epsstar_torque_L2\endcsname{112}
\expandafter\newcommand\csname pa:epsstar_torque_L3\endcsname{0}
\expandafter\newcommand\csname pa:epsstar_torque_best\endcsname{0.093}
\expandafter\newcommand\csname pa:epsstar_torque_dbest\endcsname{0.09}
\expandafter\newcommand\csname pa:epsstar_torque_half\endcsname{0.03}
\expandafter\newcommand\csname pa:f_clamped_dom\endcsname{11.8}
\expandafter\newcommand\csname pa:f_cross_Tmu17\endcsname{27.3}
\expandafter\newcommand\csname pa:f_cross_Tmu3\endcsname{20.5}
\expandafter\newcommand\csname pa:f_knee_macro\endcsname{3.56}
\expandafter\newcommand\csname pa:f_s_meas\endcsname{17.8}
\expandafter\newcommand\csname pa:f_s_model\endcsname{17.1}
\expandafter\newcommand\csname pa:free_fs_peak_hz\endcsname{17.8}
\expandafter\newcommand\csname pa:free_fws05c_macro_ff_g10\endcsname{0.486}
\expandafter\newcommand\csname pa:free_fws05c_macro_ff_g2\endcsname{0.903}
\expandafter\newcommand\csname pa:free_fws05c_macro_ff_g3\endcsname{0.651}
\expandafter\newcommand\csname pa:free_fws05c_macro_ff_g310\endcsname{0.445}
\expandafter\newcommand\csname pa:free_fws05c_macro_ff_gmid\endcsname{0.460}
\expandafter\newcommand\csname pa:free_fws05c_macro_ff_ph10\endcsname{-53.7}
\expandafter\newcommand\csname pa:free_fws05c_macro_ff_sat_pct\endcsname{0}
\expandafter\newcommand\csname pa:free_fws05c_macro_g10\endcsname{0.094}
\expandafter\newcommand\csname pa:free_fws05c_macro_g2\endcsname{0.697}
\expandafter\newcommand\csname pa:free_fws05c_macro_g3\endcsname{0.554}
\expandafter\newcommand\csname pa:free_fws05c_macro_g310\endcsname{0.419}
\expandafter\newcommand\csname pa:free_fws05c_macro_gmid\endcsname{0.527}
\expandafter\newcommand\csname pa:free_fws05c_macro_ph10\endcsname{-69.4}
\expandafter\newcommand\csname pa:free_fws05c_macro_sat_pct\endcsname{0}
\expandafter\newcommand\csname pa:free_fws05c_pisea_ff_g10\endcsname{0.806}
\expandafter\newcommand\csname pa:free_fws05c_pisea_ff_g2\endcsname{1.05}
\expandafter\newcommand\csname pa:free_fws05c_pisea_ff_g3\endcsname{1.05}
\expandafter\newcommand\csname pa:free_fws05c_pisea_ff_g310\endcsname{0.890}
\expandafter\newcommand\csname pa:free_fws05c_pisea_ff_gmid\endcsname{0.914}
\expandafter\newcommand\csname pa:free_fws05c_pisea_ff_ph10\endcsname{-9.86}
\expandafter\newcommand\csname pa:free_fws05c_pisea_ff_sat_pct\endcsname{0}
\expandafter\newcommand\csname pa:free_fws05c_pisea_fs_g3\endcsname{1.28}
\expandafter\newcommand\csname pa:free_fws05c_pisea_fs_ph3\endcsname{-163}
\expandafter\newcommand\csname pa:free_fws05c_pisea_g10\endcsname{0.790}
\expandafter\newcommand\csname pa:free_fws05c_pisea_g2\endcsname{1.04}
\expandafter\newcommand\csname pa:free_fws05c_pisea_g3\endcsname{1.03}
\expandafter\newcommand\csname pa:free_fws05c_pisea_g310\endcsname{0.847}
\expandafter\newcommand\csname pa:free_fws05c_pisea_gmid\endcsname{0.867}
\expandafter\newcommand\csname pa:free_fws05c_pisea_mu_g3\endcsname{2.25}
\expandafter\newcommand\csname pa:free_fws05c_pisea_mu_ph3\endcsname{5.09}
\expandafter\newcommand\csname pa:free_fws05c_pisea_ph10\endcsname{-12.8}
\expandafter\newcommand\csname pa:free_fws05c_pisea_sat_pct\endcsname{0}
\expandafter\newcommand\csname pa:free_fws2c_macro_ff_g10\endcsname{0.888}
\expandafter\newcommand\csname pa:free_fws2c_macro_ff_g2\endcsname{0.715}
\expandafter\newcommand\csname pa:free_fws2c_macro_ff_g3\endcsname{0.687}
\expandafter\newcommand\csname pa:free_fws2c_macro_ff_g310\endcsname{0.680}
\expandafter\newcommand\csname pa:free_fws2c_macro_ff_gmid\endcsname{0.694}
\expandafter\newcommand\csname pa:free_fws2c_macro_ff_ph10\endcsname{0.369}
\expandafter\newcommand\csname pa:free_fws2c_macro_ff_sat_pct\endcsname{0}
\expandafter\newcommand\csname pa:free_fws2c_macro_g10\endcsname{0.304}
\expandafter\newcommand\csname pa:free_fws2c_macro_g2\endcsname{0.301}
\expandafter\newcommand\csname pa:free_fws2c_macro_g3\endcsname{0.289}
\expandafter\newcommand\csname pa:free_fws2c_macro_g310\endcsname{0.254}
\expandafter\newcommand\csname pa:free_fws2c_macro_gmid\endcsname{0.261}
\expandafter\newcommand\csname pa:free_fws2c_macro_ph10\endcsname{2.53}
\expandafter\newcommand\csname pa:free_fws2c_macro_sat_pct\endcsname{0}
\expandafter\newcommand\csname pa:free_fws2c_pisea_ff_g10\endcsname{0.979}
\expandafter\newcommand\csname pa:free_fws2c_pisea_ff_g2\endcsname{0.968}
\expandafter\newcommand\csname pa:free_fws2c_pisea_ff_g3\endcsname{0.968}
\expandafter\newcommand\csname pa:free_fws2c_pisea_ff_g310\endcsname{0.943}
\expandafter\newcommand\csname pa:free_fws2c_pisea_ff_gmid\endcsname{0.943}
\expandafter\newcommand\csname pa:free_fws2c_pisea_ff_ph10\endcsname{-1.45}
\expandafter\newcommand\csname pa:free_fws2c_pisea_ff_sat_pct\endcsname{0.052}
\expandafter\newcommand\csname pa:free_fws2c_pisea_fs_g3\endcsname{0.574}
\expandafter\newcommand\csname pa:free_fws2c_pisea_fs_ph3\endcsname{142}
\expandafter\newcommand\csname pa:free_fws2c_pisea_g10\endcsname{0.784}
\expandafter\newcommand\csname pa:free_fws2c_pisea_g2\endcsname{0.732}
\expandafter\newcommand\csname pa:free_fws2c_pisea_g3\endcsname{0.681}
\expandafter\newcommand\csname pa:free_fws2c_pisea_g310\endcsname{0.698}
\expandafter\newcommand\csname pa:free_fws2c_pisea_gmid\endcsname{0.694}
\expandafter\newcommand\csname pa:free_fws2c_pisea_mu_g3\endcsname{1.20}
\expandafter\newcommand\csname pa:free_fws2c_pisea_mu_ph3\endcsname{-19.7}
\expandafter\newcommand\csname pa:free_fws2c_pisea_ph10\endcsname{-5.98}
\expandafter\newcommand\csname pa:free_fws2c_pisea_sat_pct\endcsname{5.55}
\expandafter\newcommand\csname pa:free_fws2c_pisea_tau_mu_max\endcsname{3}
\expandafter\newcommand\csname pa:free_gain_bound\endcsname{0.339}
\expandafter\newcommand\csname pa:free_ratio\endcsname{0.044}
\expandafter\newcommand\csname pa:frf_a05c_filt18_err10\endcsname{0.754}
\expandafter\newcommand\csname pa:frf_a05c_filt2_err10\endcsname{0.085}
\expandafter\newcommand\csname pa:frf_a05c_macro_err10\endcsname{0.968}
\expandafter\newcommand\csname pa:frf_a05c_macro_f3db\endcsname{3.91}
\expandafter\newcommand\csname pa:frf_a05c_macro_fmax\endcsname{15.9}
\expandafter\newcommand\csname pa:frf_a05c_macro_ftrack\endcsname{1.22}
\expandafter\newcommand\csname pa:frf_a05c_macro_g10\endcsname{0.153}
\expandafter\newcommand\csname pa:frf_a05c_macro_ph10\endcsname{-73.7}
\expandafter\newcommand\csname pa:frf_a05c_pisea_err10\endcsname{0.261}
\expandafter\newcommand\csname pa:frf_a05c_pisea_f3db\endcsname{15.9}
\expandafter\newcommand\csname pa:frf_a05c_pisea_fmax\endcsname{15.9}
\expandafter\newcommand\csname pa:frf_a05c_pisea_ftrack\endcsname{15.9}
\expandafter\newcommand\csname pa:frf_a05c_pisea_g10\endcsname{1.02}
\expandafter\newcommand\csname pa:frf_a05c_pisea_ph10\endcsname{-14.8}
\expandafter\newcommand\csname pa:frf_a2c_macro_f3db\endcsname{15.9}
\expandafter\newcommand\csname pa:frf_a2c_macro_fmax\endcsname{15.9}
\expandafter\newcommand\csname pa:frf_a2c_macro_ftrack\endcsname{3.66}
\expandafter\newcommand\csname pa:frf_a2c_macro_g10\endcsname{1.07}
\expandafter\newcommand\csname pa:frf_a2c_macro_ph10\endcsname{-53.4}
\expandafter\newcommand\csname pa:frf_a2c_pisea_f3db\endcsname{15.9}
\expandafter\newcommand\csname pa:frf_a2c_pisea_fmax\endcsname{15.9}
\expandafter\newcommand\csname pa:frf_a2c_pisea_ftrack\endcsname{14.2}
\expandafter\newcommand\csname pa:frf_a2c_pisea_g10\endcsname{1.23}
\expandafter\newcommand\csname pa:frf_a2c_pisea_ph10\endcsname{-4.81}
\expandafter\newcommand\csname pa:guard_As_freach\endcsname{6.89}
\expandafter\newcommand\csname pa:guard_Tmax\endcsname{34.6}
\expandafter\newcommand\csname pa:guard_f5\endcsname{20.1}
\expandafter\newcommand\csname pa:guard_f7\endcsname{15.7}
\expandafter\newcommand\csname pa:link_gain_bound\endcsname{0.684}
\expandafter\newcommand\csname pa:link_pend_hz\endcsname{0.790}
\expandafter\newcommand\csname pa:link_ratio\endcsname{0.161}
\expandafter\newcommand\csname pa:micro_delay_ms\endcsname{3.99}
\expandafter\newcommand\csname pa:micro_ph20\endcsname{-28.7}
\expandafter\newcommand\csname pa:mr_Bsq2noDOBlam01_eta_ss\endcsname{0.297}
\expandafter\newcommand\csname pa:mr_Bsq2noDOBlam01_lam_eta\endcsname{0.03}
\expandafter\newcommand\csname pa:mr_Bsq2noDOBlam01_taumu_ss\endcsname{0.021}
\expandafter\newcommand\csname pa:mr_Bsq2noDOBlam02_eta_ss\endcsname{0.332}
\expandafter\newcommand\csname pa:mr_Bsq2noDOBlam02_lam_eta\endcsname{0.066}
\expandafter\newcommand\csname pa:mr_Bsq2noDOBlam02_taumu_ss\endcsname{0.063}
\expandafter\newcommand\csname pa:mr_Bsq2noDOBlam05_eta_ss\endcsname{0.260}
\expandafter\newcommand\csname pa:mr_Bsq2noDOBlam05_lam_eta\endcsname{0.130}
\expandafter\newcommand\csname pa:mr_Bsq2noDOBlam05_taumu_ss\endcsname{0.129}
\expandafter\newcommand\csname pa:mr_Bsq3DOB50lam01_eta_ss\endcsname{0.00091}
\expandafter\newcommand\csname pa:mr_Bsq3DOB50lam01_lam_eta\endcsname{9.1e-05}
\expandafter\newcommand\csname pa:mr_Bsq3DOB50lam01_taumu_ss\endcsname{0.00017}
\expandafter\newcommand\csname pa:mr_eta_star_max\endcsname{0.332}
\expandafter\newcommand\csname pa:mr_eta_star_mean\endcsname{0.296}
\expandafter\newcommand\csname pa:mr_eta_star_min\endcsname{0.260}
\expandafter\newcommand\csname pa:mr_peak_KI0\endcsname{0.934}
\expandafter\newcommand\csname pa:mr_peak_KI10\endcsname{0.960}
\expandafter\newcommand\csname pa:mr_takeover_ms_KI0\endcsname{99.9}
\expandafter\newcommand\csname pa:mr_takeover_ms_KI10\endcsname{50.9}
\expandafter\newcommand\csname pa:omega_s_model\endcsname{107}
\expandafter\newcommand\csname pa:p_B_m\endcsname{12.2}
\expandafter\newcommand\csname pa:p_J_l_bar\endcsname{0.071}
\expandafter\newcommand\csname pa:p_J_l_link\endcsname{0.300}
\expandafter\newcommand\csname pa:p_J_l_link_total\endcsname{0.300}
\expandafter\newcommand\csname pa:p_J_m\endcsname{1.56}
\expandafter\newcommand\csname pa:p_K_I\endcsname{10}
\expandafter\newcommand\csname pa:p_K_ds\endcsname{160}
\expandafter\newcommand\csname pa:p_K_ps\endcsname{8000}
\expandafter\newcommand\csname pa:p_K_s\endcsname{780}
\expandafter\newcommand\csname pa:p_Kt_eff\endcsname{0.105}
\expandafter\newcommand\csname pa:p_N_d\endcsname{300}
\expandafter\newcommand\csname pa:p_T_mu_1s\endcsname{3}
\expandafter\newcommand\csname pa:p_T_mu_cont\endcsname{1.70}
\expandafter\newcommand\csname pa:p_design_J_m\endcsname{1.47}
\expandafter\newcommand\csname pa:p_design_K_s\endcsname{750}
\expandafter\newcommand\csname pa:p_design_Kt\endcsname{0.127}
\expandafter\newcommand\csname pa:p_design_T_mu_cont\endcsname{2.10}
\expandafter\newcommand\csname pa:p_design_T_mu_peak\endcsname{6.30}
\expandafter\newcommand\csname pa:p_design_tau_max\endcsname{52}
\expandafter\newcommand\csname pa:p_design_tau_max_peak\endcsname{170}
\expandafter\newcommand\csname pa:p_dt\endcsname{0.001}
\expandafter\newcommand\csname pa:p_fs\endcsname{1000}
\expandafter\newcommand\csname pa:p_lam\endcsname{0.100}
\expandafter\newcommand\csname pa:p_mgl_link\endcsname{7.40}
\expandafter\newcommand\csname pa:p_omega_Q\endcsname{50}
\expandafter\newcommand\csname pa:p_omega_mu\endcsname{4681}
\expandafter\newcommand\csname pa:p_omega_mu_hz\endcsname{745}
\expandafter\newcommand\csname pa:p_spring_allow\endcsname{38.5}
\expandafter\newcommand\csname pa:p_tau_bk\endcsname{9.03}
\expandafter\newcommand\csname pa:p_tau_c\endcsname{9.57}
\expandafter\newcommand\csname pa:p_tau_m_max_link\endcsname{150}
\expandafter\newcommand\csname pa:p_tau_m_max_locked\endcsname{100}
\expandafter\newcommand\csname pa:p_tau_max_peak\endcsname{170}
\expandafter\newcommand\csname pa:p_tau_mu_max\endcsname{3}
\expandafter\newcommand\csname pa:p_theta_s_factor\endcsname{0.900}
\expandafter\newcommand\csname pa:plane_A0p5_macro_ftrack\endcsname{1}
\expandafter\newcommand\csname pa:plane_A0p5_pisea_ftrack\endcsname{17.5}
\expandafter\newcommand\csname pa:plane_A15_macro_ftrack\endcsname{5}
\expandafter\newcommand\csname pa:plane_A15_pisea_ftrack\endcsname{5}
\expandafter\newcommand\csname pa:plane_A2_macro_ftrack\endcsname{3}
\expandafter\newcommand\csname pa:plane_A2_pisea_ftrack\endcsname{14}
\expandafter\newcommand\csname pa:plane_A5_macro_ftrack\endcsname{5}
\expandafter\newcommand\csname pa:plane_A5_pisea_ftrack\endcsname{10}
\expandafter\newcommand\csname pa:share_a05c_filt18\endcsname{28.9}
\expandafter\newcommand\csname pa:share_a05c_filt2\endcsname{78.9}
\expandafter\newcommand\csname pa:share_a05c_macro\endcsname{28.9}
\expandafter\newcommand\csname pa:share_a05c_pisea\endcsname{100}
\expandafter\newcommand\csname pa:ss_a15_macro_ftrack\endcsname{5}
\expandafter\newcommand\csname pa:ss_a15_macro_ftrip\endcsname{7}
\expandafter\newcommand\csname pa:ss_a15_macro_g1\endcsname{1.02}
\expandafter\newcommand\csname pa:ss_a15_macro_g10\endcsname{1.33}
\expandafter\newcommand\csname pa:ss_a15_macro_g5\endcsname{1.21}
\expandafter\newcommand\csname pa:ss_a15_macro_muamp1\endcsname{0}
\expandafter\newcommand\csname pa:ss_a15_macro_muamp10\endcsname{0}
\expandafter\newcommand\csname pa:ss_a15_macro_muamp5\endcsname{0}
\expandafter\newcommand\csname pa:ss_a15_macro_ph1\endcsname{-1.14}
\expandafter\newcommand\csname pa:ss_a15_macro_ph10\endcsname{-75.3}
\expandafter\newcommand\csname pa:ss_a15_macro_ph5\endcsname{-12.4}
\expandafter\newcommand\csname pa:ss_a15_macro_satmu1\endcsname{0}
\expandafter\newcommand\csname pa:ss_a15_macro_satmu10\endcsname{0}
\expandafter\newcommand\csname pa:ss_a15_macro_satmu5\endcsname{0}
\expandafter\newcommand\csname pa:ss_a15_macro_theta_pct_at_trip\endcsname{49.4}
\expandafter\newcommand\csname pa:ss_a15_pisea_ftrack\endcsname{5}
\expandafter\newcommand\csname pa:ss_a15_pisea_ftrip\endcsname{7}
\expandafter\newcommand\csname pa:ss_a15_pisea_g1\endcsname{1.000}
\expandafter\newcommand\csname pa:ss_a15_pisea_g5\endcsname{1.10}
\expandafter\newcommand\csname pa:ss_a15_pisea_muamp1\endcsname{0.479}
\expandafter\newcommand\csname pa:ss_a15_pisea_muamp5\endcsname{2.64}
\expandafter\newcommand\csname pa:ss_a15_pisea_ph1\endcsname{0.037}
\expandafter\newcommand\csname pa:ss_a15_pisea_ph5\endcsname{-5.20}
\expandafter\newcommand\csname pa:ss_a15_pisea_satmu1\endcsname{0}
\expandafter\newcommand\csname pa:ss_a15_pisea_satmu5\endcsname{19.2}
\expandafter\newcommand\csname pa:ss_a15_pisea_theta_pct_at_trip\endcsname{50.4}
\expandafter\newcommand\csname pa:ss_a5_filt18_err10\endcsname{0.765}
\expandafter\newcommand\csname pa:ss_a5_filt2_err10\endcsname{0.641}
\expandafter\newcommand\csname pa:ss_a5_macro_ceiling_f15\endcsname{5.84}
\expandafter\newcommand\csname pa:ss_a5_macro_ceiling_f20\endcsname{3.21}
\expandafter\newcommand\csname pa:ss_a5_macro_err10\endcsname{0.841}
\expandafter\newcommand\csname pa:ss_a5_macro_ftrack\endcsname{5}
\expandafter\newcommand\csname pa:ss_a5_macro_g1\endcsname{1.05}
\expandafter\newcommand\csname pa:ss_a5_macro_g10\endcsname{1.38}
\expandafter\newcommand\csname pa:ss_a5_macro_g15\endcsname{1.17}
\expandafter\newcommand\csname pa:ss_a5_macro_g20\endcsname{0.642}
\expandafter\newcommand\csname pa:ss_a5_macro_g5\endcsname{1.27}
\expandafter\newcommand\csname pa:ss_a5_macro_muamp1\endcsname{0.00078}
\expandafter\newcommand\csname pa:ss_a5_macro_muamp10\endcsname{0.00025}
\expandafter\newcommand\csname pa:ss_a5_macro_muamp15\endcsname{5.7e-05}
\expandafter\newcommand\csname pa:ss_a5_macro_muamp20\endcsname{9.3e-05}
\expandafter\newcommand\csname pa:ss_a5_macro_muamp5\endcsname{0.00028}
\expandafter\newcommand\csname pa:ss_a5_macro_ph1\endcsname{-1.85}
\expandafter\newcommand\csname pa:ss_a5_macro_ph10\endcsname{-37.2}
\expandafter\newcommand\csname pa:ss_a5_macro_ph15\endcsname{-111}
\expandafter\newcommand\csname pa:ss_a5_macro_ph20\endcsname{-147}
\expandafter\newcommand\csname pa:ss_a5_macro_ph5\endcsname{-19.5}
\expandafter\newcommand\csname pa:ss_a5_macro_satmu1\endcsname{0}
\expandafter\newcommand\csname pa:ss_a5_macro_satmu10\endcsname{0}
\expandafter\newcommand\csname pa:ss_a5_macro_satmu15\endcsname{0}
\expandafter\newcommand\csname pa:ss_a5_macro_satmu20\endcsname{0}
\expandafter\newcommand\csname pa:ss_a5_macro_satmu5\endcsname{0}
\expandafter\newcommand\csname pa:ss_a5_pisea_ceiling_f15\endcsname{5.90}
\expandafter\newcommand\csname pa:ss_a5_pisea_ceiling_f20\endcsname{3.27}
\expandafter\newcommand\csname pa:ss_a5_pisea_err10\endcsname{0.368}
\expandafter\newcommand\csname pa:ss_a5_pisea_ftrack\endcsname{10}
\expandafter\newcommand\csname pa:ss_a5_pisea_g1\endcsname{1.00}
\expandafter\newcommand\csname pa:ss_a5_pisea_g10\endcsname{1.22}
\expandafter\newcommand\csname pa:ss_a5_pisea_g15\endcsname{0.953}
\expandafter\newcommand\csname pa:ss_a5_pisea_g20\endcsname{0.472}
\expandafter\newcommand\csname pa:ss_a5_pisea_g5\endcsname{1.05}
\expandafter\newcommand\csname pa:ss_a5_pisea_muamp1\endcsname{0.279}
\expandafter\newcommand\csname pa:ss_a5_pisea_muamp10\endcsname{2.59}
\expandafter\newcommand\csname pa:ss_a5_pisea_muamp15\endcsname{2.44}
\expandafter\newcommand\csname pa:ss_a5_pisea_muamp20\endcsname{2.46}
\expandafter\newcommand\csname pa:ss_a5_pisea_muamp5\endcsname{2.34}
\expandafter\newcommand\csname pa:ss_a5_pisea_ph1\endcsname{0.067}
\expandafter\newcommand\csname pa:ss_a5_pisea_ph10\endcsname{-15.4}
\expandafter\newcommand\csname pa:ss_a5_pisea_ph15\endcsname{-90.1}
\expandafter\newcommand\csname pa:ss_a5_pisea_ph20\endcsname{-102}
\expandafter\newcommand\csname pa:ss_a5_pisea_ph5\endcsname{1.28}
\expandafter\newcommand\csname pa:ss_a5_pisea_satmu1\endcsname{0}
\expandafter\newcommand\csname pa:ss_a5_pisea_satmu10\endcsname{18.3}
\expandafter\newcommand\csname pa:ss_a5_pisea_satmu15\endcsname{19.2}
\expandafter\newcommand\csname pa:ss_a5_pisea_satmu20\endcsname{19.2}
\expandafter\newcommand\csname pa:ss_a5_pisea_satmu5\endcsname{0}
\expandafter\newcommand\csname pa:surf_A_reach\endcsname{8.38}
\expandafter\newcommand\csname pa:surf_A_spec\endcsname{10}
\expandafter\newcommand\csname pa:surf_Ks_needed\endcsname{1014}
\expandafter\newcommand\csname pa:surf_Ks_ratio\endcsname{1.30}
\expandafter\newcommand\csname pa:surf_acc_built\endcsname{109}
\expandafter\newcommand\csname pa:surf_acc_needed\endcsname{142}
\expandafter\newcommand\csname pa:surf_acc_ratio\endcsname{1.30}
\expandafter\newcommand\csname pa:surf_f_spec\endcsname{20}
\expandafter\newcommand\csname pa:surf_req_built\endcsname{4.62}
\expandafter\newcommand\csname pa:surf_req_design\endcsname{4.51}
\expandafter\newcommand\csname pa:surf_taumax_needed\endcsname{221}
\expandafter\newcommand\csname pa:tm_fund_ratio\endcsname{1.20}
\expandafter\newcommand\csname pa:tr_FS1_macro_iae\endcsname{5.94}
\expandafter\newcommand\csname pa:tr_FS1_macro_iae_hr\endcsname{0.094}
\expandafter\newcommand\csname pa:tr_FS1_macro_itae\endcsname{92.6}
\expandafter\newcommand\csname pa:tr_FS1_macro_itae_hr\endcsname{0.755}
\expandafter\newcommand\csname pa:tr_FS1_macro_n\endcsname{3}
\expandafter\newcommand\csname pa:tr_FS1_macro_rms\endcsname{0.223}
\expandafter\newcommand\csname pa:tr_FS1_macro_rms_hr\endcsname{0.0035}
\expandafter\newcommand\csname pa:tr_FS1_macro_rmsrel\endcsname{22.3}
\expandafter\newcommand\csname pa:tr_FS1_pisea_iae\endcsname{0.669}
\expandafter\newcommand\csname pa:tr_FS1_pisea_iae_hr\endcsname{0}
\expandafter\newcommand\csname pa:tr_FS1_pisea_itae\endcsname{10.2}
\expandafter\newcommand\csname pa:tr_FS1_pisea_itae_hr\endcsname{0}
\expandafter\newcommand\csname pa:tr_FS1_pisea_n\endcsname{1}
\expandafter\newcommand\csname pa:tr_FS1_pisea_rms\endcsname{0.028}
\expandafter\newcommand\csname pa:tr_FS1_pisea_rms_hr\endcsname{0}
\expandafter\newcommand\csname pa:tr_FS1_pisea_rmsrel\endcsname{2.82}
\expandafter\newcommand\csname pa:tr_FS1_pisea_throttle_s\endcsname{0}
\expandafter\newcommand\csname pa:tr_FS1_ratio\endcsname{7.92}
\expandafter\newcommand\csname pa:tr_FS2_macro_iae\endcsname{30.8}
\expandafter\newcommand\csname pa:tr_FS2_macro_iae_hr\endcsname{0.164}
\expandafter\newcommand\csname pa:tr_FS2_macro_itae\endcsname{633}
\expandafter\newcommand\csname pa:tr_FS2_macro_itae_hr\endcsname{2.99}
\expandafter\newcommand\csname pa:tr_FS2_macro_n\endcsname{3}
\expandafter\newcommand\csname pa:tr_FS2_macro_rms\endcsname{0.825}
\expandafter\newcommand\csname pa:tr_FS2_macro_rms_hr\endcsname{0.0041}
\expandafter\newcommand\csname pa:tr_FS2_macro_rmsrel\endcsname{16.5}
\expandafter\newcommand\csname pa:tr_FS2_pisea_iae\endcsname{2.07}
\expandafter\newcommand\csname pa:tr_FS2_pisea_iae_hr\endcsname{0}
\expandafter\newcommand\csname pa:tr_FS2_pisea_itae\endcsname{42.8}
\expandafter\newcommand\csname pa:tr_FS2_pisea_itae_hr\endcsname{0}
\expandafter\newcommand\csname pa:tr_FS2_pisea_n\endcsname{1}
\expandafter\newcommand\csname pa:tr_FS2_pisea_rms\endcsname{0.069}
\expandafter\newcommand\csname pa:tr_FS2_pisea_rms_hr\endcsname{0}
\expandafter\newcommand\csname pa:tr_FS2_pisea_rmsrel\endcsname{1.37}
\expandafter\newcommand\csname pa:tr_FS2_pisea_throttle_s\endcsname{0}
\expandafter\newcommand\csname pa:tr_FS2_ratio\endcsname{12.0}
\expandafter\newcommand\csname pa:tr_FS3_macro_iae\endcsname{11.8}
\expandafter\newcommand\csname pa:tr_FS3_macro_iae_hr\endcsname{0.242}
\expandafter\newcommand\csname pa:tr_FS3_macro_itae\endcsname{121}
\expandafter\newcommand\csname pa:tr_FS3_macro_itae_hr\endcsname{2.02}
\expandafter\newcommand\csname pa:tr_FS3_macro_n\endcsname{3}
\expandafter\newcommand\csname pa:tr_FS3_macro_rms\endcsname{0.739}
\expandafter\newcommand\csname pa:tr_FS3_macro_rms_hr\endcsname{0.017}
\expandafter\newcommand\csname pa:tr_FS3_macro_rmsrel\endcsname{6.16}
\expandafter\newcommand\csname pa:tr_FS3_pisea_iae\endcsname{3.41}
\expandafter\newcommand\csname pa:tr_FS3_pisea_iae_hr\endcsname{0.044}
\expandafter\newcommand\csname pa:tr_FS3_pisea_itae\endcsname{35.6}
\expandafter\newcommand\csname pa:tr_FS3_pisea_itae_hr\endcsname{0.366}
\expandafter\newcommand\csname pa:tr_FS3_pisea_n\endcsname{3}
\expandafter\newcommand\csname pa:tr_FS3_pisea_rms\endcsname{0.233}
\expandafter\newcommand\csname pa:tr_FS3_pisea_rms_hr\endcsname{0.00064}
\expandafter\newcommand\csname pa:tr_FS3_pisea_rmsrel\endcsname{1.94}
\expandafter\newcommand\csname pa:tr_FS3_pisea_throttle_s\endcsname{0}
\expandafter\newcommand\csname pa:tr_FS3_ratio\endcsname{3.18}
\expandafter\newcommand\csname pa:tr_FS3nff_macro_rms\endcsname{2.76}
\expandafter\newcommand\csname pa:tr_FS3nff_macro_rmsrel\endcsname{23.0}
\expandafter\newcommand\csname pa:tr_FS3nff_pisea_rms\endcsname{1.59}
\expandafter\newcommand\csname pa:tr_FS3nff_pisea_rmsrel\endcsname{13.2}
\expandafter\newcommand\csname pa:tr_FS3nff_pisea_throttle_s\endcsname{15.3}
\expandafter\newcommand\csname pa:w_cross_hz\endcsname{20.5}
\expandafter\newcommand\csname pa:zf_fast_macro_excluded\endcsname{0}
\expandafter\newcommand\csname pa:zf_fast_macro_n\endcsname{3}
\expandafter\newcommand\csname pa:zf_fast_macro_pk_max\endcsname{4.25}
\expandafter\newcommand\csname pa:zf_fast_macro_pk_min\endcsname{2.92}
\expandafter\newcommand\csname pa:zf_fast_macro_rms_max\endcsname{1.76}
\expandafter\newcommand\csname pa:zf_fast_macro_rms_mean\endcsname{1.71}
\expandafter\newcommand\csname pa:zf_fast_macro_rms_min\endcsname{1.65}
\expandafter\newcommand\csname pa:zf_fast_macro_v50_max\endcsname{74.4}
\expandafter\newcommand\csname pa:zf_fast_macro_v50_min\endcsname{52.8}
\expandafter\newcommand\csname pa:zf_fast_pisea_excluded\endcsname{2}
\expandafter\newcommand\csname pa:zf_fast_pisea_n\endcsname{1}
\expandafter\newcommand\csname pa:zf_fast_pisea_pk_max\endcsname{0.570}
\expandafter\newcommand\csname pa:zf_fast_pisea_pk_min\endcsname{0.570}
\expandafter\newcommand\csname pa:zf_fast_pisea_rms_max\endcsname{0.059}
\expandafter\newcommand\csname pa:zf_fast_pisea_rms_mean\endcsname{0.059}
\expandafter\newcommand\csname pa:zf_fast_pisea_rms_min\endcsname{0.059}
\expandafter\newcommand\csname pa:zf_fast_pisea_v50_max\endcsname{56.2}
\expandafter\newcommand\csname pa:zf_fast_pisea_v50_min\endcsname{56.2}
\expandafter\newcommand\csname pa:zf_fast_ratio\endcsname{28.8}
\expandafter\newcommand\csname pa:zf_slow_macro_excluded\endcsname{0}
\expandafter\newcommand\csname pa:zf_slow_macro_n\endcsname{3}
\expandafter\newcommand\csname pa:zf_slow_macro_pk_max\endcsname{1.44}
\expandafter\newcommand\csname pa:zf_slow_macro_pk_min\endcsname{0.958}
\expandafter\newcommand\csname pa:zf_slow_macro_rms_max\endcsname{0.617}
\expandafter\newcommand\csname pa:zf_slow_macro_rms_mean\endcsname{0.549}
\expandafter\newcommand\csname pa:zf_slow_macro_rms_min\endcsname{0.456}
\expandafter\newcommand\csname pa:zf_slow_macro_v50_max\endcsname{31.4}
\expandafter\newcommand\csname pa:zf_slow_macro_v50_min\endcsname{22.2}
\expandafter\newcommand\csname pa:zf_slow_pisea_excluded\endcsname{0}
\expandafter\newcommand\csname pa:zf_slow_pisea_n\endcsname{3}
\expandafter\newcommand\csname pa:zf_slow_pisea_pk_max\endcsname{0.562}
\expandafter\newcommand\csname pa:zf_slow_pisea_pk_min\endcsname{0.429}
\expandafter\newcommand\csname pa:zf_slow_pisea_rms_max\endcsname{0.059}
\expandafter\newcommand\csname pa:zf_slow_pisea_rms_mean\endcsname{0.053}
\expandafter\newcommand\csname pa:zf_slow_pisea_rms_min\endcsname{0.046}
\expandafter\newcommand\csname pa:zf_slow_pisea_v50_max\endcsname{66.1}
\expandafter\newcommand\csname pa:zf_slow_pisea_v50_min\endcsname{47.1}
\expandafter\newcommand\csname pa:zf_slow_ratio\endcsname{10.4}

\newcommand{\Nm}{N\,m}
\newcommand{\Tid}{T_{\rm id}}

\begin{document}

\title{\LARGE \bf
A Second Torque Port for Series Elastic Actuators:\\
Parallel-Integrated Design and Time-Scale Torque Allocation
}

\author{Donghao Jia$^{1}$, Xiubo Xia$^{1}$, Junhang Liu$^{1}$, Zeyu Zhu$^{1}$,
Xiaoyu Geng$^{1}$ and Jian Sun$^{1,*}$%
\thanks{$^{1}$The authors are with the School of Mechanical Engineering and Automation,
    Beihang University, Beijing 100191, China ({\tt\small donghaojia@buaa.edu.cn};
    $^{*}$corresponding author, {\tt\small buaasunjian@buaa.edu.cn}). This work received no
    institutional funding.}%
\thanks{This work has been submitted to the IEEE for possible publication. Copyright may be
    transferred without notice, after which this version may no longer be accessible.}%
}

\maketitle
\thispagestyle{empty}
\pagestyle{empty}

\begin{abstract}
A series elastic actuator has a single torque port and pays for it twice: the geared motor
must swing its own reflected inertia through the spring, so the amplitude it delivers
collapses as $\omega^{-2}$ in the command frequency $\omega$ once it saturates, while
commands below the transmission's
breakaway friction never arrive at all. This letter opens a second torque port on the load
side, placing a frameless direct-drive micro motor in parallel with a fixed-stiffness
spring---a parallel-integrated SEA, or Pi-SEA, whose delivered torque is read from spring
deflection and micro current without a sensor---and dividing the commanded torque between
the two channels by \emph{time scale} rather than by filter design. The micro torque loop
is the fast subsystem, which makes the closed loop singularly perturbed and turns the
separation the channels need into a bound to check rather than a crossover to tune; a leaky
mid-ranging integrator returns the steady load to the spring; and the amplitude ceiling,
read backwards, becomes a closed-form sizing rule that matches spring, geared motor and
micro motor to the amplitudes and frequencies an application asks for. Against SEAs, the Pi-SEA widens the tracked band at small
amplitudes and lowers the residual the joint imposes on its environment, each by an order
of magnitude.
\end{abstract}

% Keywords are entered in the PaperCept submission form; ieeeconf has no
% \begin{IEEEkeywords} environment and the RA-L template prints none.
% Series elastic actuator, compliant joints, force control, actuator design,
% singular perturbation.

\section{INTRODUCTION}
\label{sec:intro}

Series elastic actuators (SEAs) buy force control with a spring: a
calibrated deflection replaces the torque sensor, the spring bounds what a collision
transmits, and disturbance-observer designs have made the torque loop
accurate~\cite{oh2017hpforce,ugurlu2022benchmark}. The motor's own inertia sends the
bill. Torque reaching the load is first stored in the spring, so the geared motor must
swing its reflected inertia through the deflection the command asks for; the amplitude it
can still deliver therefore falls as \(\omega^{-2}\) in the command frequency \(\omega\)
once its torque saturates. We call
this the \emph{amplitude ceiling}: the large-torque bandwidth limit derived in
Sec.~\ref{subsec:sizing}, of which maximum-torque-transmissibility analysis gives a
closed-loop counterpart~\cite{leeoh2022mtt}, and one that controller gain does not lift. The prototype below,
held to \N{p_tau_m_max_locked}\,\Nm\ on the macro channel by the software clamp under which
the bandwidth campaign was run, delivers a measured \N{ceil_macro_f20_meas}\,\Nm\ at
20\,Hz under that clamp. Softening the spring lowers the ceiling;
stiffening it recovers bandwidth by
removing the compliance the spring exists for~\cite{lee2021codsea}. At the
other end of the amplitude range the transmission that reflects the inertia breaks away at
\N{p_tau_bk}\,\Nm, and commands well below that are not delivered at all. One spring, one
motor, one port---the two failures have one cause, and both grow with the reduction ratio.

The most developed response makes the compliance adjustable: variable-stiffness
actuators reshape the torque--deflection characteristic with a second motor and a dedicated
transmission~\cite{li2024vsareview,yu2025compactvsa}, adjustable-equilibrium designs take
that freedom in discrete steps~\cite{chen2026aepea}, and parallel designs place the spring
alongside a direct-drive motor rather than behind
it~\cite{mathews2023pvsa}. That literature names the
cost: existing designs generally do not let both motors contribute to output
power at once~\cite{zhang2025dmlsvsa}, so the second motor's torque never reaches the load,
while the mechanism adds mass and volume and rate-limits the stiffness
itself~\cite{yu2025compactvsa}. Damping, the other
mechanical knob~\cite{monteleone2022damping,loeffl2025vea}, suppresses the resonance the
ceiling leaves but cannot supply the torque it denies. In all of them the load torque still crosses one elastic port: the ceiling moves
with \(K_s\) but stays.

The control side is equally bounded. Elastic-structure-preserving control and its
robust extension~\cite{keppler2018espi,lee2024resp} and adaptive-robust torque
control~\cite{dai2025artc} enlarge what one motor renders, and the same family bounds it by
the actuator behind the spring~\cite{keppler2021asrlimits}. Singular perturbation has
reached compliant-robot control with the elastic mode as the fast
subsystem~\cite{kim2020spimpedance,shi2025spsynthesis}---a reading of the same
single-actuator plant, not a second source of torque.

The remedy that acts on the cause is a second torque port bypassing the spring.
Parallel-coupled micro--macro actuation~\cite{morrell1998micromacro} and the distributed
macro--mini architecture~\cite{zinn2004dm2,sardellitti2007dm2,shin2010hybrid,kim2017hybrid}
put a small motor at the load beside a large compliant one and let it render what the
macro channel leaves of the command; the pairing has since been built with
brakes~\cite{dills2021hybrid,chaichaowarat2022macromini}, a second geared
motor~\cite{yasuda2014sea2act} and an elastic aerial suspension~\cite{yigit2023macromini},
with its stability under interaction studied~\cite{gosselin2023macromini}. In these designs
one torque is known only indirectly, through a gear or cable, a transducer or a spring
measured away from the load, and three things the designer needs never appear: a
condition, rather than a filter corner, on how far apart in speed the two channels must be;
a bound on the share the fast motor keeps at DC, which mid-ranging provides in process
control~\cite{nowak2024midranging} but redundancy resolution and predictive optimization
do not; and a closed-form size for the small motor from the ceiling the geared macro's
own inertia sets, rather than a design procedure or an optimization over
efficiency and mass~\cite{khorasani2022redundant}.

This letter keeps one fixed-stiffness spring and adds a frameless direct-drive micro motor
coaxial with the joint, stator on the housing and rotor on the link
(Figs.~\ref{fig:concept} and~\ref{structure}), forming a \emph{parallel-integrated SEA}
(Pi-SEA). With no transmission the micro torque is read from its current and the spring
torque from its deflection, so the delivered torque---and with it the torque the joint
exerts on whatever holds the load---is known without a transducer, and the allocation acts
on delivered rather than commanded torque. The micro torque loop is then the fast
subsystem and the spring--load interaction the slow one, so how fast the second actuator
must be is not a filter corner to tune but a bound to check, \(\epsilon^\star\) of
Theorem~\ref{thm:spt_composite}. A leaky mid-ranging law returns the steady load to the
spring with a proved floor, and the amplitude ceiling of the elastic channel, read
backwards, sizes the micro motor in closed form. The scope is one joint, a
regulation-scope bound that depends on the state scaling, and no bandwidth advantage
claimed where the residual exceeds the micro motor's short-term peak.

Section~\ref{sec:spt} develops the allocation and its stability,
Sec.~\ref{sec:prototype} the sizing rule and the prototype, and
Sec.~\ref{sec:validation} the experiments.

\section{DUAL-TIME-SCALE TORQUE ALLOCATION}
\label{sec:spt}

\subsection{Dynamic Model and Time-Scale Separation}
\label{subsec:dynamic_model_spt}

\begin{figure}[t]
\centering
\includegraphics[width=\columnwidth]{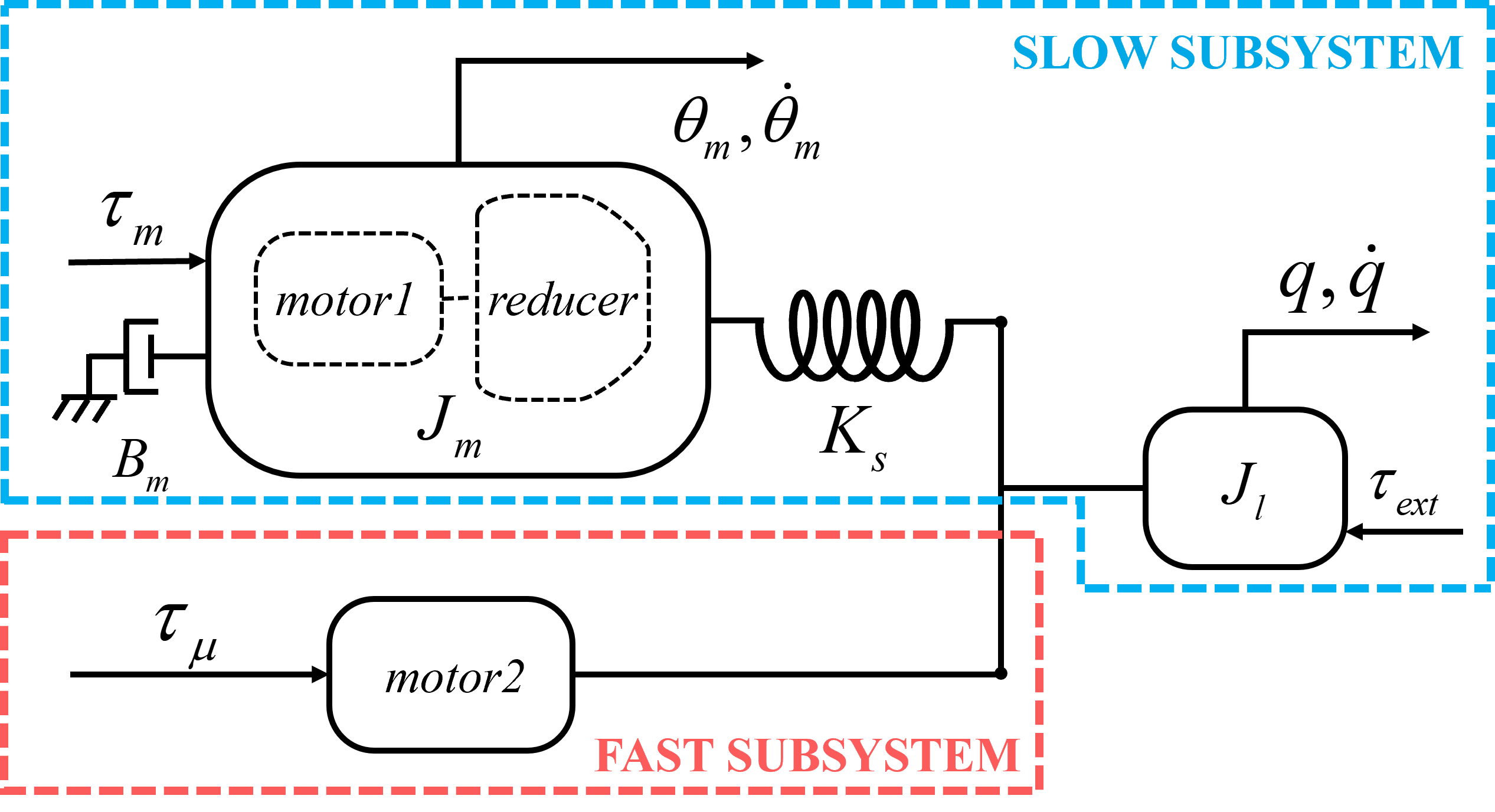}
\caption{Single-joint model of the parallel-integrated SEA. The slow subsystem is the
compliant SEA--load interaction, the fast subsystem the torque loop of the load-side micro
motor, whose torque reaches the load without passing through the spring or the reduction.}
\label{fig:concept}
\end{figure}

Fig.~\ref{fig:concept} shows the single-joint model. With \(\theta_s=\theta_m-q\) the
spring deflection and \(F_s=K_s\theta_s\) the torque it carries, flange and load obey
\begin{equation}
\begin{aligned}
    J_m \ddot{\theta}_m + B_m \dot{\theta}_m &= \tau_m - F_s + d_m , \\
    J_l \ddot q &= F_s + \tau_\mu + \tau_{\rm ext} + d_l ,
\end{aligned}
\label{eq:joint_model}
\end{equation}
with \(\theta_m\), \(q\) the flange and load positions, \(J_m\), \(B_m\) the reflected
motor inertia and damping, \(J_l\) the load inertia, \(\tau_m\), \(\tau_\mu\) the macro and
micro torques, \(\tau_{\rm ext}\) the external torque, and \(d_m\), \(d_l\) the lumped
disturbances, \(d_m\) dominantly friction. Unlike conventional SEA
singular-perturbation analyses, which take the elastic mode as the fast
dynamics~\cite{kim2020spimpedance}, \(\theta_s\) stays a slow state here.

Write \(\omega_s\) and \(\omega_\mu\) for the characteristic bandwidths of the SEA--load
dynamics and of the micro motor's inner torque loop. The micro motor carries no reduction
and no spring, so \(\omega_\mu\gg\omega_s\) by construction, and their ratio
\begin{equation}
    \epsilon=\omega_s/\omega_\mu\ll1
    \label{eq:eps_def}
\end{equation}
is the small parameter of the analysis: the micro torque loop is the fast subsystem, the
SEA--load interaction the slow one. Over the band of interest that inner loop is a
first-order lag,
\begin{equation}
    \frac{\tau_\mu(s)}{\tau_\mu^\star(s)}=\frac{a_\mu}{\epsilon s+a_\mu},
    \qquad
    a_\mu>0,
    \label{eq:micro_fast_model}
\end{equation}
with \(\tau_\mu^\star\) the commanded micro torque and \(a_\mu\) the loop pole written in
the slow time scale, so that \(a_\mu/\epsilon=\omega_\mu\). We call the series-elastic
channel the \emph{macro} actuator and the second motor the \emph{micro} actuator.

\subsection{Force Controller}
\label{subsec:force_controller_spt}

\begin{figure}[t]
\centering
\includegraphics[width=\columnwidth]{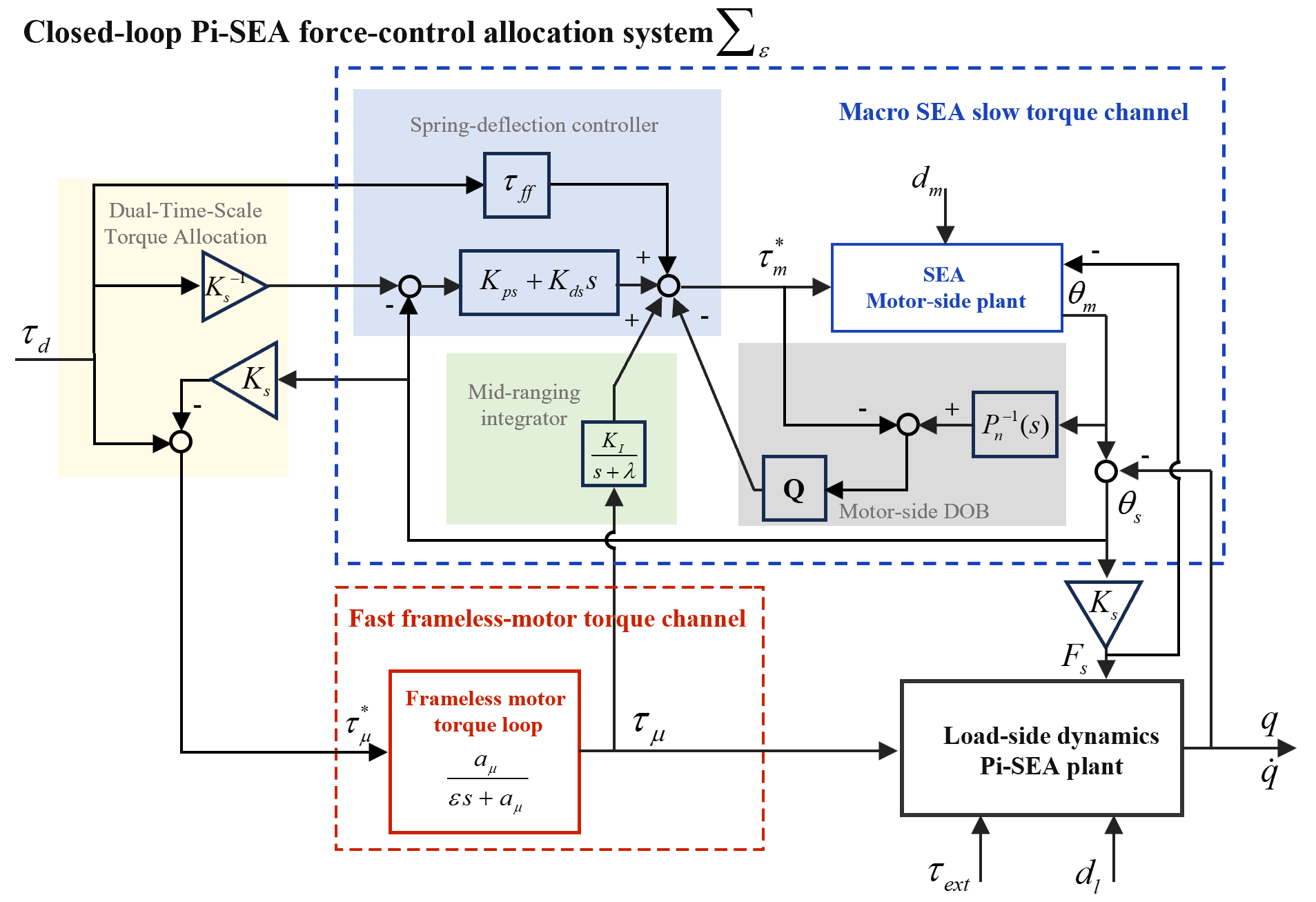}
\caption{Dual-time-scale allocation: the macro channel is force-controlled on the spring
deflection with a leaky mid-ranging integral, the micro channel renders the residual
\(\tau_d-F_s\), and both act on the common load.}
\label{block}
\end{figure}

Let \(\tau_d\) be the commanded load-side torque. The macro channel is force-controlled on
the spring deflection (Fig.~\ref{block}),
\begin{equation}
    \tau_m^\star
    =
    \tau_{\rm ff}
    +
    K_{ps}(\theta_s^d-\theta_s)
    +
    K_{ds}(\dot{\theta}_s^d-\dot{\theta}_s)
    +
    K_I \eta
    -
    \hat d_m ,
    \label{eq:macro_target}
\end{equation}
with the deflection set-point \(\theta_s^d=\tau_d/K_s\), deflection gains
\(K_{ps},K_{ds}>0\) and mid-ranging gain \(K_I>0\). Of the remaining two terms,
\(\hat d_m\) is the output of the motor-side disturbance observer drawn in
Fig.~\ref{block}: following~\cite{oh2017hpforce} it treats the measured spring torque as an
external input, so its nominal inverse plant is \(P_n^{-1}(s)=J_ms^2+B_ms\) and the
residual it leaves is \(\tilde d_m=d_m-\hat d_m\). The feedforward \(\tau_{\rm ff}\) is the
command scaled by the share of the macro torque that reaches the load: unity when the load
is held, and
\begin{equation}
    \tau_{\rm ff}=\Big(1+\frac{J_m}{J_l}\Big)\tau_d
    \label{eq:ff_inertia}
\end{equation}
for a free inertial load, since below the SEA--load resonance the spring is quasi-rigid and
only \(J_l/(J_m+J_l)\) of the macro torque arrives. A gravity term restores the load about
its own equilibrium, so \eqref{eq:ff_inertia} applies only well above the pendulum
frequency \(\sqrt{mgl/J_l}\). The mid-ranging
state is the leaky integral of the micro torque,
\begin{equation}
    \dot\eta=\tau_\mu-\lambda\eta,
    \qquad
    \lambda>0 ,
    \label{eq:mid_ranging_leaky}
\end{equation}
implemented with anti-windup saturation. Through \(K_I\eta\) it transfers whatever the
micro motor carries at low frequency to the macro channel, and the leak bounds it,
\(\|\eta\|_\infty\le\|\eta(0)\|e^{-\lambda t}+\|\tau_\mu\|_\infty/\lambda\), whether or not
the loop reaches a constant equilibrium.

The micro channel renders the residual between the command and what the spring delivers,
\begin{equation}
    \tau_\mu^\star=h(x,\tau_d):=\tau_d-K_s\theta_s .
    \label{eq:micro_target}
\end{equation}
The joint's output torque is \(\Tid=F_s+\tau_\mu\), so
\begin{equation}
    \Tid-\tau_d=\tau_\mu-h(x,\tau_d)=:y ,
    \label{eq:output_error}
\end{equation}
the \emph{output torque error equals the tracking error of the fast channel}: the macro
loop and the mid-ranging integral decide how the delivered torque is \emph{split}, not
whether it is delivered.

Define the slow and fast states
\begin{equation}
    x =
    \begin{bmatrix}
        q & \dot q & \theta_s & \dot{\theta}_s & \eta
    \end{bmatrix}^{T}\!\in\mathbb{R}^5,
    \qquad
    z=\tau_\mu\in\mathbb{R} .
    \label{eq:slow_fast_states}
\end{equation}
The closed loop is then in standard singularly perturbed form,
\begin{align}
    \dot x &= f(x,z,\tau_d),
    \label{eq:spt_slow}\\
    \epsilon \dot z &= -a_\mu z+a_\mu h(x,\tau_d),
    \label{eq:spt_fast}
\end{align}
which separates the slow SEA--load interaction from the fast micro torque tracking.

\subsection{Singular-Perturbation Stability}
\label{subsec:spt_stability}

The analysis is for the nominal closed loop with constant \(\tau_d=\tau_d^\star\),
\(\tilde d_m=0\) and \(d_l=0\); bounded residual disturbances and slowly varying commands
enter as inputs and give local practical stability.

\subsubsection{Reduced and boundary-layer subsystems}
Setting \(\epsilon=0\) in \eqref{eq:spt_fast} gives the slow manifold \(z=h(x,\tau_d)\), on
which \(\tau_\mu=\tau_\mu^\star\) and, by \eqref{eq:output_error}, \(\Tid=\tau_d\).
Substituting into \eqref{eq:spt_slow} yields the reduced system
\begin{equation}
    \Sigma_0:\qquad \dot x=f_r(x,\tau_d):=f(x,h(x,\tau_d),\tau_d),
    \label{eq:reduced_system}
\end{equation}
the SEA--load loop under an ideal micro torque source. With the fast error
\(y=z-h(x,\tau_d)\) and the stretched time \(\sigma=t/\epsilon\), the boundary-layer system
for frozen \(x\) is
\begin{equation}
    \Sigma_b:\qquad \frac{dy}{d\sigma}=-a_\mu y ,
    \label{eq:boundary_layer_system}
\end{equation}
exponentially stable for \(a_\mu>0\); \(q\) stays with the slow dynamics, and only the
convergence of \(\tau_\mu\) to its command is fast.

\subsubsection{Composite stability}
Throughout, \(f\) and \(h\) are continuously differentiable and locally Lipschitz near the
equilibrium, with both actuators away from their torque limits. The one
condition that is not automatic is the stability of the reduced loop.

\begin{assumption}[Reduced slow-subsystem stability]
\label{ass:slow_stable}
For constant \(\tau_d^\star\), the reduced system \eqref{eq:reduced_system} has an
equilibrium \(x^\star\) with \(h(x^\star,\tau_d^\star)=\lambda\eta^\star\), and admits a
Lyapunov function \(V_s(x)\) satisfying the quadratic sandwich and decay conditions
\eqref{eq:Vs_bound_appendix}--\eqref{eq:Vs_decay_appendix} with constants
\(c_1,c_2,\alpha_s>0\). For the linearized macro loop this holds whenever the
Routh--Hurwitz condition \eqref{eq:macro_RH_appendix} on \((K_{ps},K_{ds},K_I,\lambda)\) is
met, which also upper-bounds \(K_I\) (Remark~\ref{rem:KI_tuning}).
\end{assumption}

\begin{theorem}[Local exponential stability]
\label{thm:spt_composite}
Under Assumption~\ref{ass:slow_stable} there exists \(\epsilon^\star>0\), given
explicitly by \eqref{eq:epsilon_star_appendix}, such that the equilibrium
\((x,y)=(x^\star,0)\) of \eqref{eq:spt_slow}--\eqref{eq:spt_fast} is locally
exponentially stable for all \(0<\epsilon<\epsilon^\star\).
\end{theorem}

For time-varying commands the fast error picks up
\(-\epsilon(\partial h/\partial\tau_d)\dot\tau_d\), so by \eqref{eq:output_error} the
output torque error is proportional to the command rate.

\subsubsection{Steady-state allocation}
At any constant equilibrium \eqref{eq:mid_ranging_leaky} gives \(\dot\eta=0\), hence
\begin{equation}
    \lim_{t\to\infty}\tau_\mu(t)=\lambda\eta^\star ,
    \label{eq:floor}
\end{equation}
and, since \(y\to0\), \(h(x^\star,\tau_d^\star)=\lambda\eta^\star\): the micro motor renders
the transient and sheds the steady load to the spring channel up to a floor proportional
to the leak. Since \(\eta^\star\) is set by the macro channel's stick--slip equilibrium
and its integral gain, not by \(\lambda\), \eqref{eq:floor} predicts a \(\tau_\mu^\star\)
linear in \(\lambda\). Whether that floor can be driven to zero by \(\lambda\to0\) is a
separate question this letter does not settle.

\subsection{Frequency-Domain Reading of the Allocation}
\label{subsec:freq_allocation}

Previous designs specify their split in the frequency
domain~\cite{morrell1998micromacro,zinn2004dm2}; here it is a closed-loop property.
Linearizing about the equilibrium of Theorem~\ref{thm:spt_composite}, with
\(H_m(s)=F_s(s)/\tau_d(s)\) the transfer from command to spring torque,
\begin{equation}
    \frac{\tau_\mu}{\tau_d}=\frac{a_\mu}{\epsilon s+a_\mu}\,\big(1-H_m(s)\big),
    \qquad
    \frac{\Tid}{\tau_d}=H_m+\frac{\tau_\mu}{\tau_d},
    \label{eq:alloc_spectrum}
\end{equation}
so the micro channel is the exact complement of what the spring channel fails to deliver,
low-passed by the micro loop, with DC content \(O(\lambda)\)---the frequency-domain face of
\eqref{eq:floor}. A filter divides the \emph{command}, \eqref{eq:micro_target} the
\emph{delivered} torque, so this crossover migrates to wherever the spring channel stops
delivering.

\section{DESIGN RULE AND PROTOTYPE}
\label{sec:prototype}

\subsection{Design Rule: Spring, Macro and Micro Actuator}
\label{subsec:sizing}

The three hardware quantities of the joint---the spring stiffness \(K_s\), the macro
motor's output-side acceleration \(\tau_m^{\max}/J_m\), and the micro torque
\(T_\mu\)---are tied together by the large-torque envelope of the two channels, without
tuning weights. To deliver a load torque of amplitude \(A\) at frequency \(\omega\)
through the spring, the macro motor must swing its own reflected inertia by \(A/K_s\),
which costs \(AJ_m\omega^2/K_s\) of motor torque; bounded by \(\tau_m^{\max}\), the macro
channel can supply at most \(A_{\rm macro}(\omega)=\tau_m^{\max}K_s/(J_m\omega^2)\), a
\(-40\)\,dB/dec ceiling that no controller gain moves, since it is the torque the command
asks of the motor whatever the control law. This is why a
modest output specification buys a large joint: at 20\,Hz the \(\omega^2\) factor is
\(1.6\times10^4\), so a few \Nm\ at the load consumes a 100:1 joint's whole peak torque,
nearly all of it turning the reflected inertia round. The micro
channel drives the load directly and contributes a flat \(T_\mu\). Requiring their sum to
cover the specification gives
\begin{equation}
    T_\mu \;\ge\; A-\frac{K_s}{\omega^{2}}\,\frac{\tau_m^{\max}}{J_m} ,
    \label{eq:sizing_rule}
\end{equation}
in which the figure of merit \(\tau_m^{\max}/J_m\) degrades as \(1/N\) in the reduction
ratio (\(\tau_m^{\max}\propto N\), \(J_m\propto N^2\)), the penalty that motivates
quasi-direct-drive actuation~\cite{wensing2017proprioceptive}. Only the \(\omega^2\) term
is kept, so \eqref{eq:sizing_rule} is the asymptote above the coupled spring--load
resonance \(\omega_{\rm res}=\sqrt{K_s(J_m^{-1}+J_l^{-1})}\); past
\(\omega_{\rm res}/\sqrt2\) it overstates what the command asks of the macro motor, and so
sizes \(T_\mu\) conservatively. Read the other way,
\eqref{eq:sizing_rule} is the force bandwidth a given hardware set reaches at amplitude
\(A\),
\begin{equation}
    f_{\max}(A)=\frac{1}{2\pi}\sqrt{\frac{\tau_m^{\max}K_s}{J_m\,(A-T_\mu)}}
    \quad\text{for } A>T_\mu ,
    \label{eq:f_of_A}
\end{equation}
while for \(A\le T_\mu\) the micro loop bandwidth \(\omega_\mu\) sets the reach. The
regimes meet at \(\omega_{\rm cross}=\sqrt{\tau_m^{\max}K_s/(J_mT_\mu)}\): below it
\eqref{eq:sizing_rule} sizes the micro motor, above it the largest command the micro motor
carries alone does.

\subsection{Specification, Selection and Prototype}
\label{subsec:selection}

Voluntary motion in daily activities lies below 2\,Hz and pathological tremor concentrates
in 3--12\,Hz, with components reported up to
17.3\,Hz~\cite{nguyen2021tremororthosis}. The joint's own design target is
\(A=\N{surf_A_spec}\)\,\Nm\ at \N{surf_f_spec}\,Hz, above that band. The macro joint is a
rigid integrated unit, harmonic drive 100:1, rated \N{p_design_tau_max}\,\Nm\ continuous
and \N{p_tau_max_peak}\,\Nm\ peak, data-sheet reflected inertia
\N{p_design_J_m}\,kg\,m$^2$, with a spring designed for \N{p_design_K_s}\,\Nm/rad.
\eqref{eq:sizing_rule} is a large-torque envelope held for the duration of a command, so it
is evaluated at the peak throughout this letter, and it asks for
\(T_\mu\ge\N{des_T_mu_req}\)\,\Nm; a frameless torque motor, MOSRAC U9424DJA01 rated
\N{p_design_T_mu_cont}/\N{p_design_T_mu_peak}\,\Nm\ continuous/peak at
\(k_t=\N{p_design_Kt}\)\,\Nm/A, meets it with a \N{des_margin_peak}\(\times\) margin. Over
the \((K_s,\tau_m^{\max}/J_m)\) plane, \eqref{eq:sizing_rule} becomes the design surface of
Fig.~\ref{fig:surface}, which makes the same selection for any joint.

\begin{figure}[t]
\centering
\includegraphics[width=0.78\columnwidth]{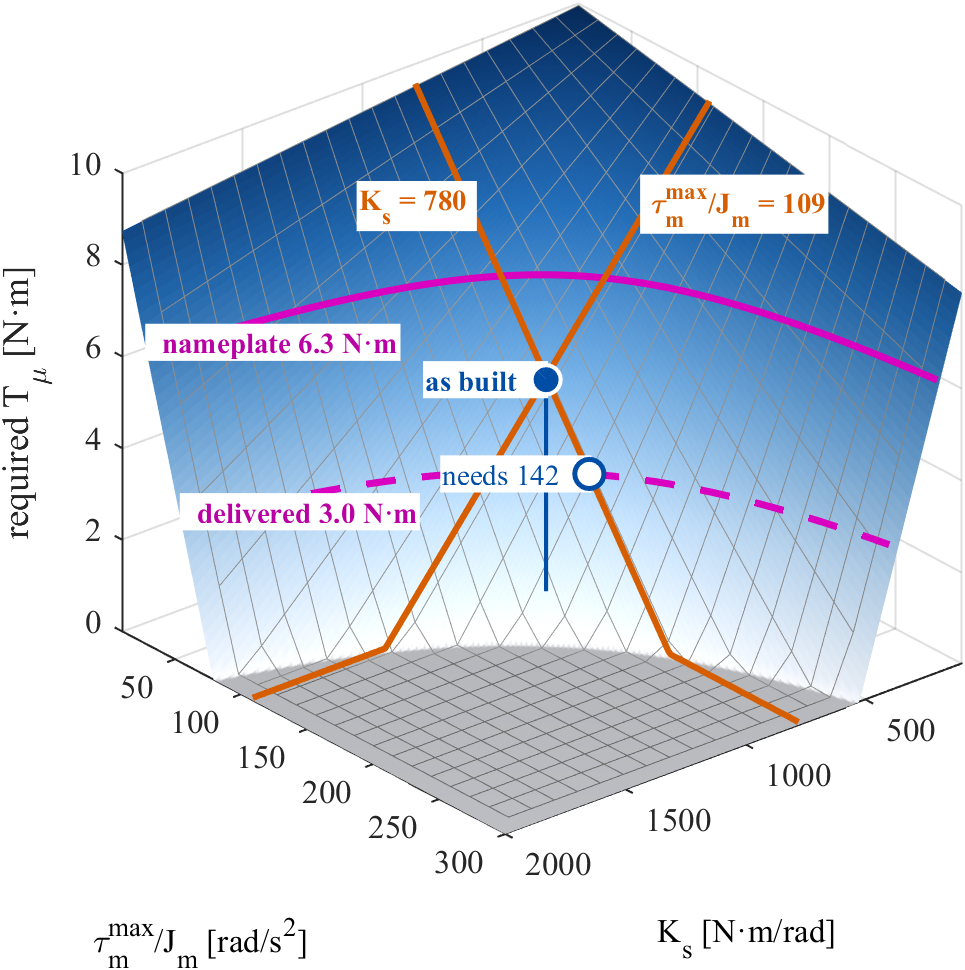}
\caption{The sizing rule \eqref{eq:sizing_rule} as a design surface for the
\N{surf_A_spec}\,\Nm\ at \N{surf_f_spec}\,Hz specification. Height and colour are the micro
torque \(T_\mu\) the rule requires at each \((K_s,\tau_m^{\max}/J_m)\), falling to zero
where the macro channel already covers the specification alone; a motor serves everything
\emph{below} its own magenta iso-line. The orange curves section the surface at the
as-built joint; the shortfall they show is a prediction at the joint's peak torque,
quantified in the text.}
\label{fig:surface}
\end{figure}

\begin{figure*}[t]
\centering
\begin{minipage}[b]{0.403\textwidth}\centering
\includegraphics[height=1.95in]{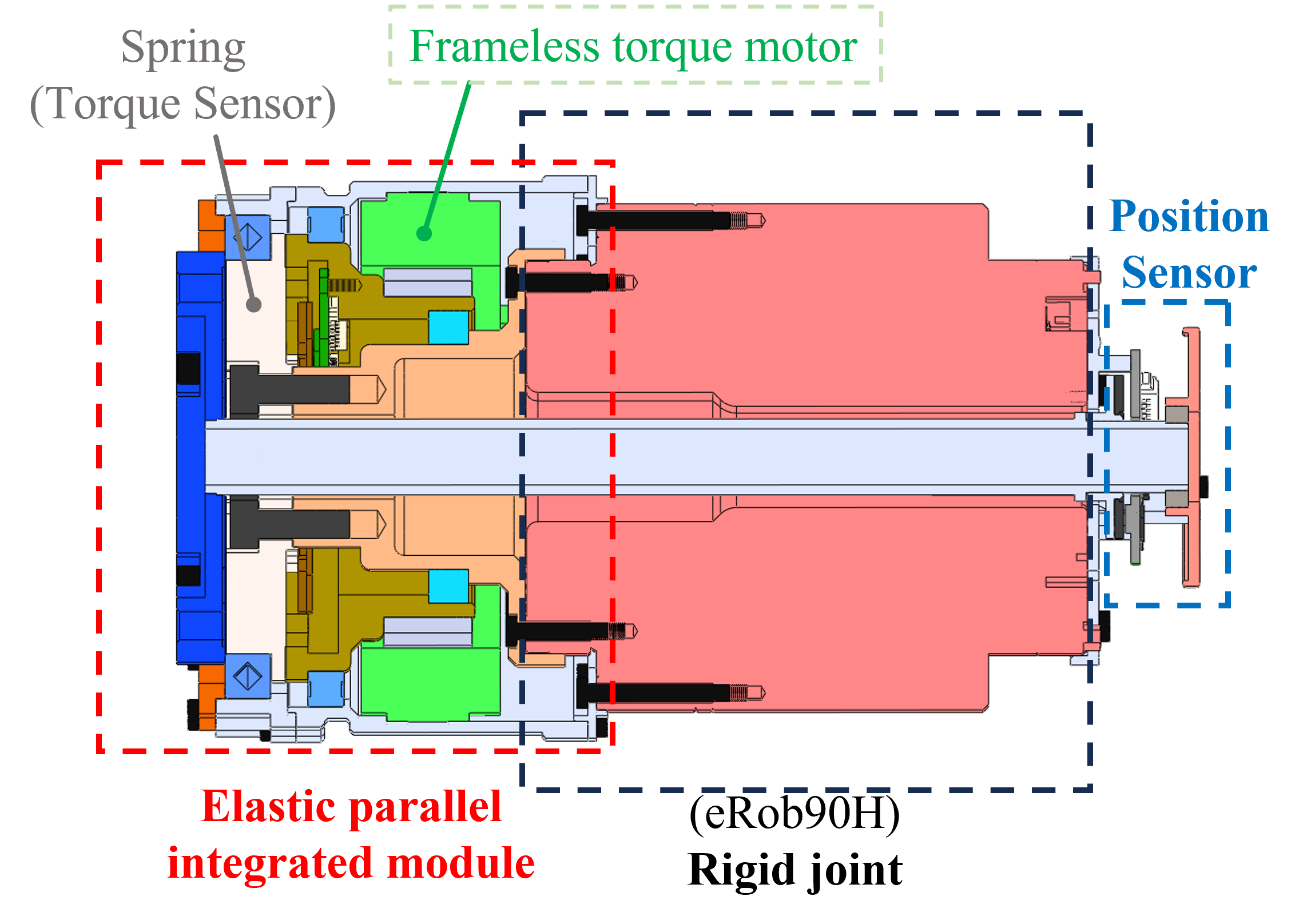}\\[1pt]{\footnotesize (a)}
\end{minipage}\hfill
\begin{minipage}[b]{0.590\textwidth}\centering
\includegraphics[height=1.95in]{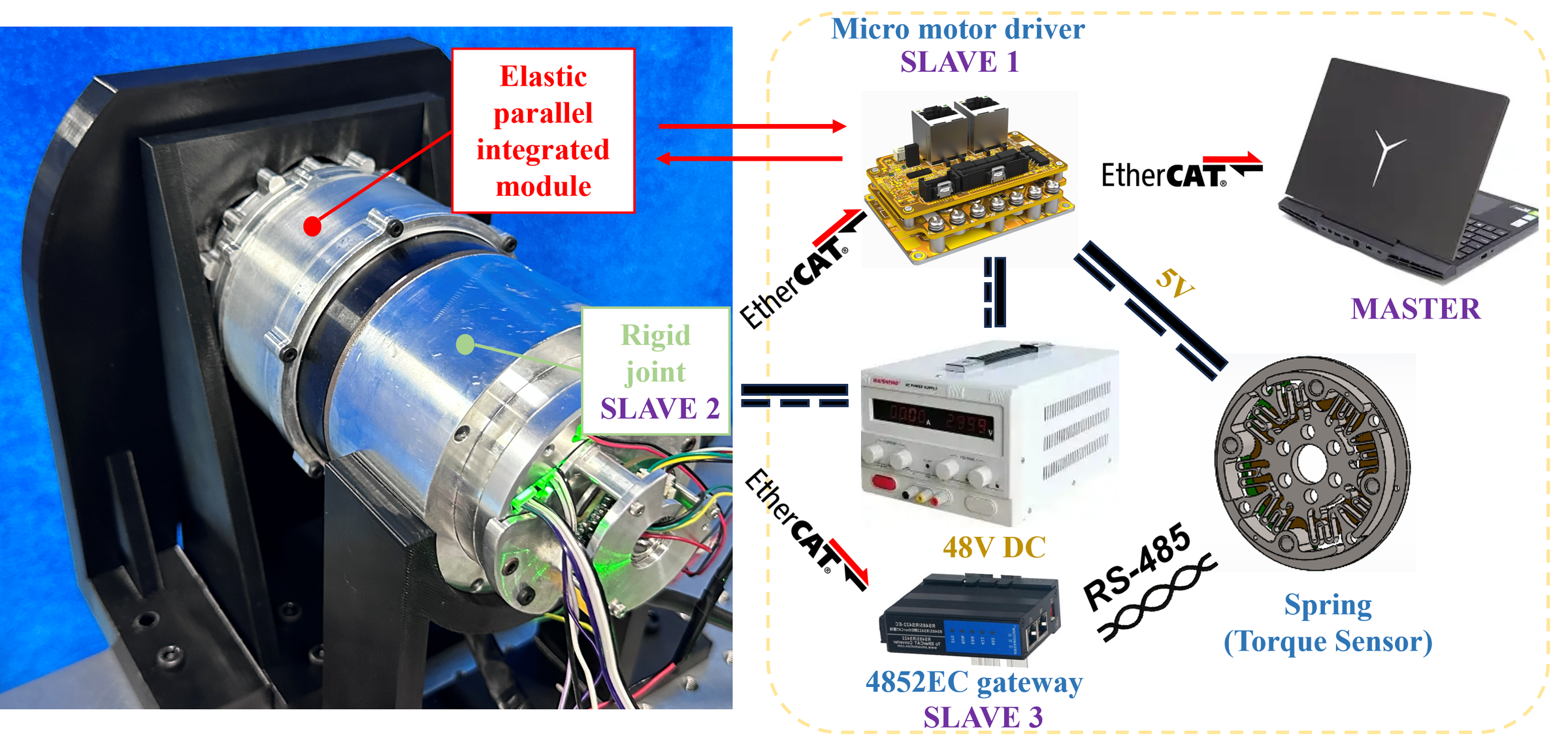}\\[1pt]{\footnotesize (b)}
\end{minipage}
\caption{The prototype as built. (a) Cross-section: the torsion spring (\(K_s\)) and the
frameless micro motor (\(\tau_\mu\)) act in parallel on the load, and an absolute encoder
across the spring measures \(\theta_s\). (b) The assembled joint and its bench. One
EtherCAT master closes both loops at \N{p_fs}\,Hz over three slaves; the spring doubles as
the torque sensor, so no external transducer is in the loop.}
\label{structure}
\end{figure*}

The prototype (Fig.~\ref{structure}) mounts the parallel module on the joint flange: a
torsion spring connects the flange to the link, and the frameless motor, stator on the
housing and rotor on the link, injects \(\tau_\mu\) without a transmission. \(\theta_m\)
and \(q\) come from the joint and micro-rotor encoders. Identified parameters are listed in
Table~\ref{tab:params}; the micro drive's \(I^2t\) protection holds the usable micro torque
to \N{p_T_mu_cont}/\N{p_T_mu_1s}\,\Nm\ continuous/peak. The as-built point of
Fig.~\ref{fig:surface} therefore sits above the delivered iso-line, missing the
specification by a factor \N{surf_acc_ratio} on \(\tau_m^{\max}/J_m\): at the joint's
\N{p_tau_max_peak}\,\Nm\ peak the rule predicts \N{surf_A_reach}\,\Nm\ at \N{surf_f_spec}\,Hz,
or 10\,\Nm\ up to \N{blt_f_reach_3_peak}\,Hz, while the campaign below runs under a
\N{p_tau_m_max_locked}\,\Nm\ clamp. Its crossover \(\omega_{\rm cross}/2\pi=\N{f_cross_Tmu3}\)\,Hz
lies beyond the 15--16\,Hz at which the closed macro loop runs out of phase, so the micro
motor works in the small-amplitude regime.

\begin{table}[t]
\centering
\caption{Parameters identified on the assembled prototype and the gains and load
conditions used.}
\label{tab:params}
\footnotesize
\setlength{\tabcolsep}{3.5pt}
% Generated by tools/paperA_figures.m.  Do not edit by hand.
\begin{tabular}{@{}llll@{}}
\toprule
Symbol & Value & Symbol & Value \\
\midrule
$J_m$ [kg\,m$^2$] & 1.56 & $K_{ps}$ [N\,m/rad] & 8000 \\
$B_m$ [N\,m\,s/rad] & 12.2 & $K_{ds}$ [N\,m\,s/rad] & 160 \\
$\tau_b$ [N\,m] & 9.0 & $K_I$, $\lambda$ [1/s] & 10, 0.1 \\
$K_s$ [N\,m/rad] & 780 & $N_d$, $\omega_Q$ [rad/s] & 300, 50 \\
$k_t$ (meas.) [N\,m/A] & 0.105 & $J_l$ bar / link [kg\,m$^2$] & 0.071 / 0.30 \\
$\omega_\mu/2\pi$ [Hz] & 745 & $a_\mu$ [rad/s] & 112 \\
$T_\mu$ cont./peak [N\,m] & 1.7/3.0 &  &  \\
\bottomrule
\end{tabular}

\end{table}

\section{EXPERIMENTAL VALIDATION}
\label{sec:validation}

\subsection{Setup, Baseline and Metrics}
\label{subsec:exp_setup}

All experiments use Table~\ref{tab:params} and the \N{p_tau_m_max_locked}\,\Nm\ macro clamp
unless another is named. The macro-only baseline is the same controller with the micro
motor disabled and the gains unchanged, so the two differ in nothing but the second
channel: with \(\tau_\mu\equiv0\) the allocation degenerates to a feed-forward,
deflection-PD and disturbance-observer force loop of the family benchmarked
in~\cite{oh2017hpforce,ugurlu2022benchmark}. A second baseline replaces
\eqref{eq:micro_target} by a fixed first-order split of the \emph{command}, cornered at
2\,Hz or at \N{alloc_model_fx}\,Hz. The feed-forward \eqref{eq:ff_inertia} is used on the
free bar and at 4\,Hz on the gravity link, \(\tau_{\rm ff}=\tau_d\) below the link's
\N{link_pend_hz}\,Hz pendulum. No external torque sensor is used: the delivered torque is
\(\Tid=K_s\theta_s+\tau_\mu\), with \(\tau_\mu\) the measured current times the measured
\(k_t\) of Table~\ref{tab:params}, read back \N{micro_delay_ms}\,ms after the command over
EtherCAT, from the readback phase at 20\,Hz; the 1\,kHz bus schedule puts about two cycles of
that before the torque acts and two after it, so the logged \(\Tid\) trails the physical
torque by about 2\,ms, \(7^\circ\) at 10\,Hz. A command is
\emph{tracked} where \(|1-\Tid/\tau_d|\le0.5\), a \(29^\circ\) phase lag at unity gain. The accompanying video shows the hand-push and gravity-link runs.

\subsection{Time-Scale Separation: \(\epsilon\) against \(\epsilon^\star\)}
\label{subsec:exp_eps}

The two time scales are \(\omega_\mu/2\pi=\N{p_omega_mu_hz}\)\,Hz for the micro torque loop,
its \(-45^\circ\) point measured on a jig, and \(\omega_s/2\pi=\N{f_s_meas}\)\,Hz for the SEA--load resonance, read from the free-bar
chirp of Sec.~\ref{subsec:exp_free} against \N{f_s_model}\,Hz from the identified
parameters, so \(\epsilon=\N{eps}\); the clamped loop's dominant poles, at
\N{f_clamped_dom}\,Hz, give a smaller \(\epsilon=\N{eps_clamped}\), so this choice is
conservative. The EtherCAT lag of Sec.~\ref{subsec:exp_setup} is a transport delay outside
the first-order model \eqref{eq:micro_fast_model} and does not enter \(\epsilon\). Evaluated
for the clamped load---the case in which
Assumption~\ref{ass:slow_stable} holds without an added load damping term
(Remark~\ref{rem:KI_tuning})---in the torque-equivalent state
\((K_s\theta_s,\,K_s\dot\theta_s/\omega_s,\,K_I\eta)\),
the appendix expression \eqref{eq:epsilon_star_appendix} gives
\(\epsilon^\star=\N{epsstar_torque_best}\). The
prototype therefore runs at \(\epsilon/\epsilon^\star=\N{eps_ratio_torque}\), inside
Theorem~\ref{thm:spt_composite} in these coordinates.

\subsection{Force Bandwidth and Allocation with the Load Clamped}
\label{subsec:exp_locked}

\begin{figure*}[tb]
\centering
\includegraphics[width=0.95\textwidth]{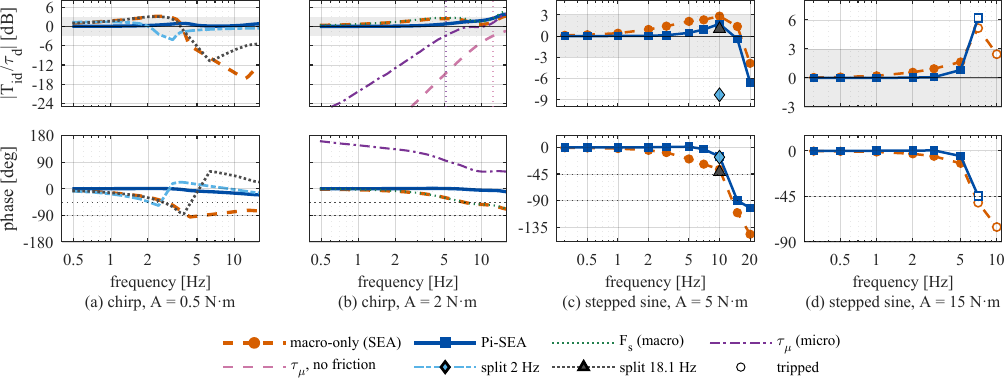}
\caption{Closed-loop \(\Tid/\tau_d\), load clamped, identical gains. (a) 0.5\,\Nm\ and
(b) 2\,\Nm\ chirps, 0.3--20\,Hz, drawn where the coherence exceeds 0.5; (a) adds the fixed
command split at 2 and 18.1\,Hz, (b) the Pi-SEA split into \(F_s\) and \(\tau_\mu\), with
\(\tau_\mu\) under the frictionless model dashed and both \(-3\)\,dB points dotted---the gap
is the allocation responding to delivered torque, which no fixed filter of \(\tau_d\) does.
(c) 5\,\Nm, with the splits at 10\,Hz, and (d) 15\,\Nm\ stepped sines; hollow markers were
cut short by the spring-energy guard.}
\label{fig:frf_locked}
\end{figure*}

\begin{figure}[t]
\centering
\includegraphics[width=\columnwidth]{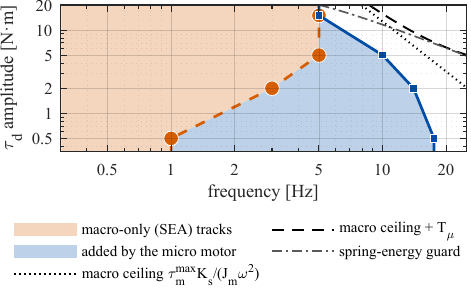}
\caption{Where each controller tracks the command (\(|1-\Tid/\tau_d|\le0.5\)) with the load
clamped, under the \(\tau_m^{\max}=\N{p_tau_m_max_locked}\)\,\Nm\ software clamp of this
campaign, not the joint's \N{p_tau_max_peak}\,\Nm\ peak. Orange, the region the macro-only
loop already tracks; blue, the band the micro motor adds; white, neither. Markers are the
four measured amplitudes, the 0.5\,\Nm\ Pi-SEA edge from stepped sines past the chirp; the
boundaries between them are interpolated. Dotted, the macro
amplitude ceiling \(\tau_m^{\max}K_s/(J_m\omega^2)\) of \eqref{eq:sizing_rule}; dashed,
that ceiling plus \(T_\mu\); dash--dot, the spring-energy guard of the test rig, which is
what stops the 15\,\Nm\ row at 7\,Hz.}
\label{fig:plane}
\end{figure}

Three regimes appear (Figs.~\ref{fig:frf_locked} and~\ref{fig:plane}), divided by the macro
loop's linear window \(A=\tau_bK_s/K_{ps}=\N{A_lin_window}\)\,\Nm, below which its
proportional term cannot reach the breakaway torque. Above the window, at 2 and 5\,\Nm,
both loops hold unity gain and what the second port buys is \emph{phase}: at 10\,Hz from
\N{frf_a2c_macro_ph10}$^\circ$ to \N{frf_a2c_pisea_ph10}$^\circ$ and from
\N{ss_a5_macro_ph10}$^\circ$ to \N{ss_a5_pisea_ph10}$^\circ$, widening the tracked band
from \N{frf_a2c_macro_ftrack} to \N{frf_a2c_pisea_ftrack}\,Hz and from
\N{ss_a5_macro_ftrack} to \N{ss_a5_pisea_ftrack}\,Hz. That phase is lost accelerating
\(J_m\) through the spring and no gain returns it. By the \(-3\)\,dB criterion the loops are
alike there, neither falling 3\,dB within the chirp's coherent band at 2\,\Nm\ and both
holding to 15\,Hz at 5\,\Nm; the tracked band differs by phase alone. Below the window the
second port buys \emph{authority} instead,
and buys the most there: at 0.5\,\Nm\ the tracked edge moves from
\N{frf_a05c_macro_ftrack}\,Hz to between \N{a05s_edge_lo} and \N{a05s_edge_hi}\,Hz, and
the \(-3\)\,dB point from \N{frf_a05c_macro_f3db}\,Hz to beyond that band, limited
by the micro path's \N{micro_delay_ms}\,ms lag, which \(\Tid\), built from the current
readback, cannot separate from a reporting delay. At 15\,\Nm\ it buys nothing, and none is
claimed: both controllers stop at \N{ss_a15_macro_ftrip}\,Hz on the rig's spring-energy
guard.

Two comparisons close the section. Against a fixed command split at identical gains,
Pi-SEA tracks the whole 0.5\,\Nm\ chirp where the 2 and \N{alloc_model_fx}\,Hz splits track
\N{share_a05c_filt2} and \N{share_a05c_filt18}\,\% of it---a filter divides the command and
loses wherever the macro channel does not deliver its share, which is also why the micro
channel here is loaded from \N{alloc_f3dB_mu}\,Hz against \N{alloc_model_f3dB_mu}\,Hz for
the frictionless model of \eqref{eq:alloc_spectrum}. Repeating the large-amplitude points
with the macro clamp at 50 and 75\,\Nm\ confirms the ceiling law itself: the delivered
spring torque scales with the clamp and falls as \(\omega^{-2}\), and
\eqref{eq:sizing_rule} with its spring term restored fits all \N{ceil_npts} points within
\(\N{ceil_spring_err_min}\) to \(+\N{ceil_spring_err_max}\)\,\%.

\subsection{Mid-Ranging}
\label{subsec:exp_midranging}

\begin{figure}[t]
\centering
\includegraphics[width=\columnwidth]{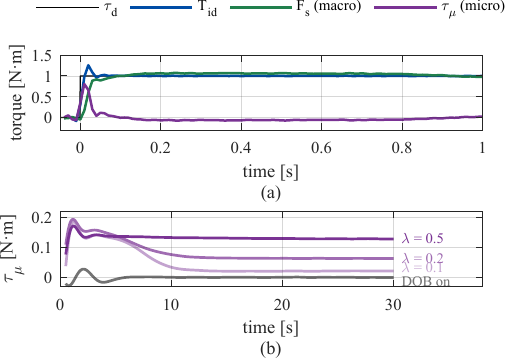}
\caption{Leaky mid-ranging, load clamped. (a) 1\,\Nm\ step with \(K_I=10\): the micro motor
carries the edge and the spring channel takes the load over. (b) Micro torque after the same
step with the DOB off for \(\lambda=0.1,0.2,0.5\)\,s$^{-1}$ (purple) and with the DOB on for
\(\lambda=0.1\) (grey).}
\label{fig:midranging}
\end{figure}

\begin{figure*}[tb]
\centering
\includegraphics[width=0.95\textwidth]{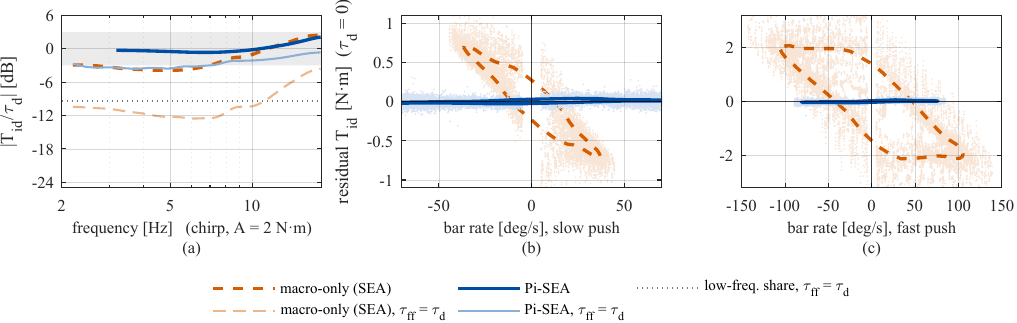}
\caption{Free symmetric bar (\(J_l=\N{p_J_l_bar}\)\,kg\,m$^2$). (a) Closed-loop gain
\(|\Tid/\tau_d|\) of the 2\,\Nm\ chirp with \eqref{eq:ff_inertia} (dark) and with
\(\tau_{\rm ff}=\tau_d\) (light); dotted, the low-frequency share of the latter. (b), (c) Residual output torque under \(\tau_d=0\) while the bar is pushed by
hand, slow and fast, against the bar rate: raw samples faint, fitted loops dark (median
torque in equal-count rate bins on each branch of the stroke).}
\label{fig:free}
\end{figure*}

Mid-ranging shortens the hand-over of a 1\,\Nm\ step from the micro motor to the spring
from \N{mr_takeover_ms_KI0} to \N{mr_takeover_ms_KI10}\,ms (to 10\,\% of peak) and leaves it a steady load of
\N{mr_Bsq3DOB50lam01_taumu_ss}\,\Nm\ (Fig.~\ref{fig:midranging}), the disturbance
observer having removed the friction the integral would otherwise carry.

Switching the observer off exposes the floor \eqref{eq:floor} itself: the micro
torque then settles at \N{mr_Bsq2noDOBlam01_taumu_ss}, \N{mr_Bsq2noDOBlam02_taumu_ss} and
\N{mr_Bsq2noDOBlam05_taumu_ss}\,\Nm\ for \(\lambda=0.1\), 0.2 and 0.5\,s$^{-1}$, against
the predicted \(\lambda\eta^\star\) of \N{mr_Bsq2noDOBlam01_lam_eta},
\N{mr_Bsq2noDOBlam02_lam_eta} and \N{mr_Bsq2noDOBlam05_lam_eta}\,\Nm. The three
\(\eta^\star\) agree to \N{mr_eta_star_min}--\N{mr_eta_star_max}\,\Nm\,s: the floor is
linear in \(\lambda\) at a slope set by the stick torque and \(K_I\). 

\subsection{Free Bar: Authority on a Light Load}
\label{subsec:exp_free}

On a free bar only \(r=J_l/(J_m+J_l)=\N{free_ratio}\) of the macro torque reaches the load,
so the macro-only loop delivers just \N{free_fws2c_macro_gmid} of a 2\,\Nm\ command over
2--10\,Hz. The feed-forward \eqref{eq:ff_inertia} recovers most of that deficit and the
micro motor, acting on the bar directly, closes the rest: Pi-SEA delivers
\N{free_fws2c_pisea_ff_g310} over 3--10\,Hz (Fig.~\ref{fig:free}(a)).

\subsection{Zero-Torque Command and Torque Tracking on a Gravity Link}
\label{subsec:exp_tracking}

Under \(\tau_d=0\) the residual \(\Tid\) is what the joint imposes on a hand that moves it.
The second port cuts it \N{zf_slow_ratio}-fold on slow pushes, from
\N{zf_slow_macro_rms_min}--\N{zf_slow_macro_rms_max} to
\N{zf_slow_pisea_rms_min}--\N{zf_slow_pisea_rms_max}\,\Nm\ rms, and by a larger factor on
fast ones (Fig.~\ref{fig:free}(b),(c)). What remains sits at the stroke reversals, where
the macro channel turns its inertia round against its own friction.

\begin{figure*}[tb]
\centering
\includegraphics[width=0.95\textwidth]{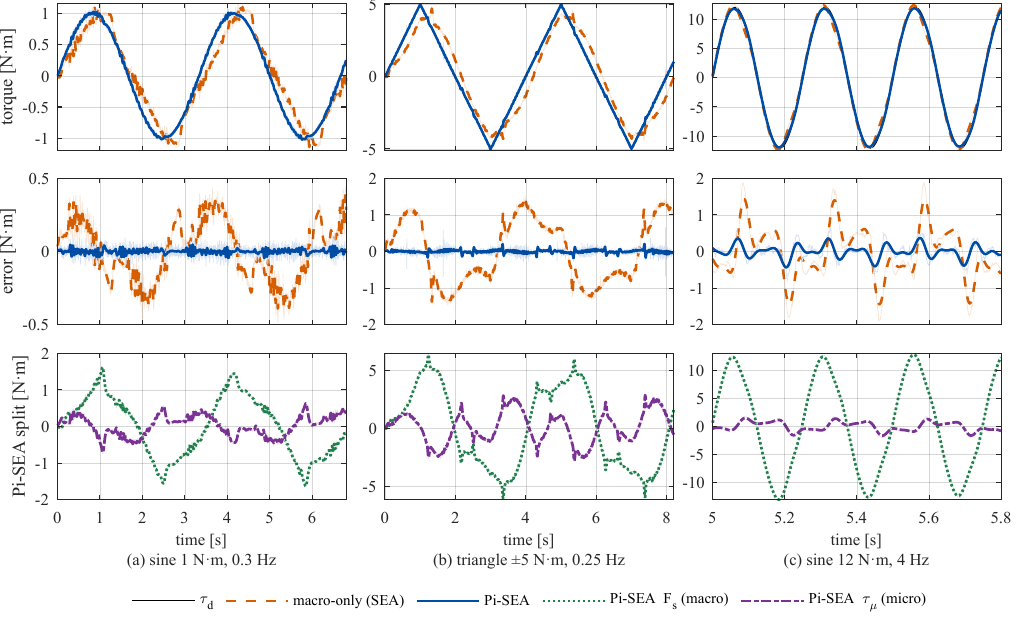}
\caption{Torque tracking on the gravity link (bar with 2.5\,kg at 0.3\,m, hanging start):
(a) sine 1\,\Nm\ at 0.3\,Hz, (b) triangle \(\pm5\)\,\Nm\ at 0.25\,Hz, (c) sine 12\,\Nm\ at
4\,Hz with \eqref{eq:ff_inertia}. Rows: command and delivered torque; error \(\tau_d-\Tid\) (raw faint, filtered
dark); Pi-SEA split into \(F_s\) and \(\tau_\mu\). One run of each controller is shown.}
\label{fig:track}
\end{figure*}

On the gravity link the second port cuts the rms torque error \N{tr_FS1_ratio}-fold on a
1\,\Nm\ sine at 0.3\,Hz and \N{tr_FS2_ratio}-fold on a \(\pm5\)\,\Nm\ triangle at
0.25\,Hz---the waveform whose constant-slew zero crossing exposes the macro dead band
(Fig.~\ref{fig:track}). On a 12\,\Nm\ sine at 4\,Hz, where \eqref{eq:ff_inertia} is active
and the macro channel already carries its share, the remaining gain is
\N{tr_FS3_ratio}-fold.

\section{CONCLUSION}
\label{sec:conclusion}

This letter has presented a complete design and control treatment of a joint that places a
direct-drive micro motor in parallel with a series-elastic macro channel. Allocating by
time scale rather than by filter turns the separation the two channels need into a bound
that can be checked against the hardware; a leaky mid-ranging law returns the steady load
to the spring and leaves the fast channel its range; and the amplitude ceiling of the
elastic channel, read backwards, sizes the micro motor in closed form. Experiments on a
single joint support all three: against a macro-only baseline at identical gains the second
port restores the phase the spring path loses, opens a tracked band below the breakaway
window where the macro channel delivers nothing, and cuts the zero-torque interaction
residual by an order of magnitude. What differs from earlier macro--mini pairings is that
the crossover is a closed-loop property rather than a filter to be designed, and that the
micro motor bypasses the elastic path variable-stiffness designs instead reshape, and that
both torques are read where they act, the spring's from its deflection and the micro's
from its current, so the joint carries no torque sensor. The cost
is a second drive and its thermal budget, the micro rotor itself adding negligible inertia;
that budget is also what bounds the gains at large amplitude. The limits are one joint and
a regulation-scope bound that depends on the state scaling; the coupled multi-DoF case,
stability in contact with an environment, and a thermal model in place of the fixed clamp
are left as future work.

\appendix
\section{SINGULAR-PERTURBATION CONDITIONS AND COMPOSITE STABILITY PROOF}
\label{app:spt_proof}

The Lyapunov conditions of Assumption~\ref{ass:slow_stable} are: there exist a
continuously differentiable \(V_s(x)\) and constants \(c_1,c_2,\alpha_s>0\) with
\begin{align}
    c_1\|x-x^\star\|^2
    &\le
    V_s(x)
    \le
    c_2\|x-x^\star\|^2,
    \label{eq:Vs_bound_appendix}\\
    \frac{\partial V_s}{\partial x}f_r(x,\tau_d^\star)
    &\le
    -\alpha_s\|x-x^\star\|^2 .
    \label{eq:Vs_decay_appendix}
\end{align}

\begin{remark}[Practical tuning condition]
\label{rem:KI_tuning}
Neglecting \(\tilde d_m\) and treating the load-side motion as a bounded slow signal
(exactly so when clamped), the deflection loop with the mid-ranging integral has the cubic
characteristic polynomial \(J_ms^3+(B_m+K_{ds}+J_m\lambda)s^2+
\big(K_s+K_{ps}+\lambda(B_m+K_{ds})\big)s+K_IK_s+\lambda(K_{ps}+K_s)\), whose coefficients
are all positive for \(\lambda,K_I>0\), so Routh--Hurwitz reduces to
\begin{equation}
    K_I
    <
    \frac{(B_m+K_{ds})\big[K_{ps}+K_s+\lambda(B_m+K_{ds})+\lambda^2J_m\big]}
         {J_mK_s} ,
    \label{eq:macro_RH_appendix}
\end{equation}
The leak relaxes the bound and keeps \(\eta\) bounded. For a free load the reduced system has the micro motor supplying
\(\tau_d-F_s\) exactly, so the macro loop no longer damps \(q\) and the assumption requires
a damped or gravity-restored load; the bound is evaluated for the clamped load in
Sec.~\ref{subsec:exp_eps}.
\end{remark}

\begin{IEEEproof}
The construction is the standard composite-Lyapunov
one~\cite[Thm.~11.4]{khalil2002nonlinear}. With \(\tilde x=x-x^\star\), \(y=z-h(x,\tau_d^\star)\) and
\(\Delta f=f(x,y+h,\tau_d^\star)-f(x,h,\tau_d^\star)\), Lipschitz continuity gives
\(L_1>0\) with
\begin{equation}
    \frac{\partial V_s}{\partial x}\Delta f(x,y)
    \le
    L_1\|\tilde x\||y| .
    \label{eq:L1_appendix}
\end{equation}
In the original time scale the fast error obeys
\(\epsilon\dot y=-a_\mu y-\epsilon(\partial h/\partial x)f(x,y+h,\tau_d^\star)\), and local
boundedness provides \(L_2,L_3>0\) with
\begin{equation}
    -y
    \frac{\partial h}{\partial x}
    f(x,y+h,\tau_d^\star)
    \le
    L_2\|\tilde x\||y|+L_3y^2 .
    \label{eq:L23_appendix}
\end{equation}
\(L_1,L_2\) measure how the micro torque error feeds back into the SEA--load motion
through \(h\), \(L_3\) its self-coupling. With \(h=\tau_d-K_s\theta_s\),
\(\partial h/\partial x\) is the constant row \([0\;0\;{-K_s}\;0\;0]\); since \(\tau_\mu\)
drives only \(\eta\), on which \(h\) does not depend, \(L_3=(\partial h/\partial x)B=0\)
exactly. Normalising \(V_s\) by \(A_r^{T}P+PA_r=-I\) fixes \(\alpha_s=1\), and
\(a_\mu=\omega_s\) makes \(a_\mu/\epsilon=\omega_\mu\); the torque-equivalent scaling gives
\(L_1=\N{epsstar_torque_L1}\), \(L_2=\N{epsstar_torque_L2}\). Taking \(V_c=(1-d)V_s(x)+\tfrac{d}{2}y^2\), \(d\in(0,1)\),
and using \eqref{eq:Vs_decay_appendix}, \eqref{eq:L1_appendix} and \eqref{eq:L23_appendix},
\begin{equation}
    \dot V_c
    \le
    -
    \begin{bmatrix}
        \|\tilde x\|\\
        |y|
    \end{bmatrix}^{T}
    \begin{bmatrix}
        (1-d)\alpha_s & -\tfrac{1}{2}\beta\\[2pt]
        -\tfrac{1}{2}\beta & d\left(\tfrac{a_\mu}{\epsilon}-\gamma\right)
    \end{bmatrix}
    \begin{bmatrix}
        \|\tilde x\|\\
        |y|
    \end{bmatrix},
    \label{eq:Vc_dot_bound_appendix}
\end{equation}
with \(\beta=(1-d)L_1+dL_2\) and \(\gamma=L_3\). The matrix is positive definite precisely
when \((1-d)\alpha_s\,d\,(a_\mu/\epsilon-\gamma)>\tfrac14\beta^2\), that is, for
\begin{equation}
    0<\epsilon<\epsilon^\star
    =
    \frac{a_\mu}
    {
        \gamma+
        \dfrac{\beta^2}{4d(1-d)\alpha_s}
    } .
    \label{eq:epsilon_star_appendix}
\end{equation}
so \(\dot V_c\) is locally negative definite and \((x^\star,0)\) locally exponentially
stable. The best weight \(d\) is \(1/2\) only if \(L_1=L_2\)
(\(d=\N{epsstar_torque_dbest}\) here). The bound depends on the state scaling through
\(V_s,L_1,L_2\): the torque-equivalent coordinates of Sec.~\ref{subsec:exp_eps} put every
slow state in the unit of the fast error, \Nm, so \(\epsilon^\star\) is a sufficient bound
in those coordinates, not an invariant of the plant, and carries its coordinates with it.
\end{IEEEproof}

\bibliography{ref}
\bibliographystyle{IEEEtran}

\end{document}